\pdfoutput=1
\documentclass[11pt]{article}
\usepackage[preprint]{acl}
\usepackage{times}
\usepackage{latexsym}
\usepackage[T1]{fontenc}
\usepackage[utf8]{inputenc}
\usepackage{microtype}
\usepackage{inconsolata}
\usepackage{graphicx}
\usepackage{booktabs}       
\usepackage{amsfonts}       
\usepackage{nicefrac}       
\usepackage{microtype}      
\usepackage{xcolor}         
\usepackage{amsmath}
\usepackage{amssymb}
\usepackage{mathtools}
\usepackage{amsthm}
\usepackage{wrapfig}
\usepackage{algorithm}
\usepackage{algorithmic}
\usepackage{tcolorbox}
\usepackage{colortbl}
\usepackage{enumitem}
\usepackage{marvosym}
\usepackage[capitalize,noabbrev]{cleveref}
\DeclareMathOperator*{\ours}{ReCo}

\theoremstyle{plain}
\ifx\theorem\undefined
\newtheorem{theorem}{Theorem}

\newtheorem{proposition}{Proposition}

\theoremstyle{definition}
\newtheorem{definition}{Definition}

\title{Accelerating Training of Autoregressive Video Generation Models \\via Local Optimization with Representation Continuity}

\author{Yucheng Zhou, ~Jianbing Shen$^{\text{\Letter}}$\\
SKL-IOTSC, CIS, University of Macau \\
\texttt{yucheng.zhou@connect.um.edu.mo, jianbingshen@um.edu.mo}
}

\begin{document}
\maketitle
\renewcommand{\thefootnote}{\Letter} 
\footnotetext{Corresponding Author.}

\begin{abstract}
Autoregressive models have shown superior performance and efficiency in image generation, but remain constrained by high computational costs and prolonged training times in video generation. 
In this study, we explore methods to accelerate training for autoregressive video generation models through empirical analyses. 
Our results reveal that while training on fewer video frames significantly reduces training time, it also exacerbates error accumulation and introduces inconsistencies in the generated videos.
To address these issues, we propose a Local Optimization (Local Opt.) method, which optimizes tokens within localized windows while leveraging contextual information to reduce error propagation. 
Inspired by Lipschitz continuity, we propose a Representation Continuity (ReCo) strategy to improve the consistency of generated videos. ReCo utilizes continuity loss to constrain representation changes, improving model robustness and reducing error accumulation.
Extensive experiments on class- and text-to-video datasets demonstrate that our approach achieves superior performance to the baseline while halving the training cost without sacrificing quality.
\end{abstract}

\section{Introduction}
Existing visual generative models based on the diffusion model demonstrate excellent visual generation capabilities \citep{LVDM,Skip,OpenSora}.
Recently, many studies \citep{LlamaGen,MAR,songbroad,zhoucondition} explore the potential of autoregressive language models in image generation, discovering that they offer advantages over diffusion-based models in inference speed and performance.
Furthermore, these models show significant potential for integration with large language models (LLMs), enabling the development of large multimodal models that unify understanding and generation capability \citep{Chameleon}.

\begin{figure}[t]
    \centering
    \includegraphics[width=1\linewidth]{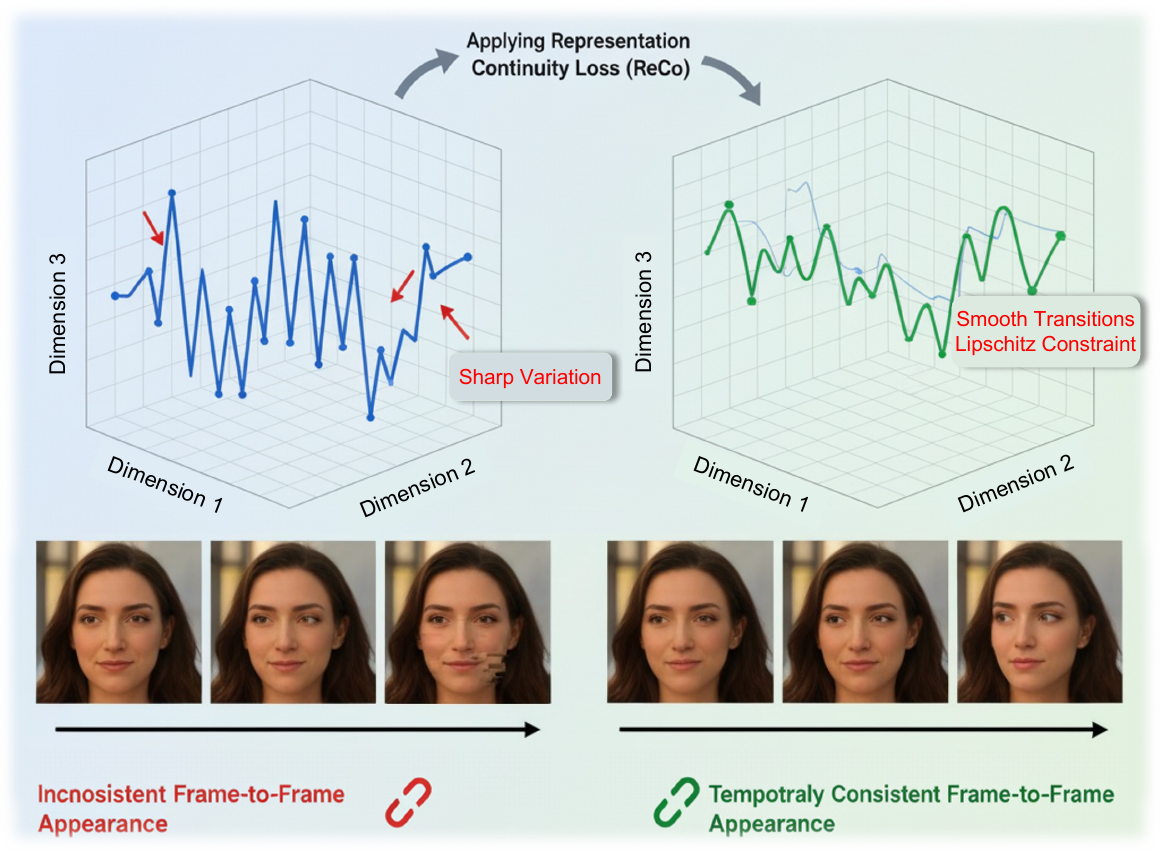}
    \vspace{-4mm}
    \caption{\small \textbf{Left:} Autoregressive models exhibit abrupt representation changes, causing temporal inconsistencies. \textbf{Right:} Our ReCo enforces smooth transitions via continuity loss, yielding temporally consistent videos.}
    \label{fig:intro}
    \vspace{-2mm}
\end{figure}

Despite the success of autoregressive models in image generation, recent research focuses on extending autoregressive models to video generation \citep{Loong,Emu3,VRC}. 
Autoregressive video generation models, modeling on discrete tokens from Vector Quantized Variational Autoencoder (VQVAE \citep{OmniTokenizer}), demonstrate promising performance. 
However, these models encounter challenges related to computational efficiency, due to the significantly longer video token sequences compared to image token sequences, which increase both training and inference costs.
To reduce the inference cost, \citet{VRC} make the first attempt to accelerate autoregressive video generation by vision representation compression. 
Despite their success, autoregressive models for video generation continue to encounter significant challenges in training, i.e., high computational costs and long training times.

In this study, we conduct extensive experiments to investigate the training acceleration of autoregressive video generation. 
First, we explore a Fewer-Frames method, where the model can be trained on sequences with fewer frames, and more video frames are generated iteratively during inference.
Although this approach reduces training time relative to the baseline, it paradoxically increases inference time significantly.
Moreover, it underperforms the baseline due to inconsistencies in the generated videos.
We theoretically demonstrate that the Fewer-Frames model exhibits greater accumulated error compared to the baseline. 
Further empirical validation is provided by evaluating the quality and consistency of its generated video frames using two evaluation metrics, i.e., PSNR \citep{PSNR} and Optical Flow \citep{OpticalFlow}.

To alleviate the inconsistency issue on the Fewer-Frames method, we propose the Local Optimization (Local Opt.) method. 
This approach optimizes tokens within a window, including frame blocks, while considering preceding tokens as context.
Random placement and overlapping of windows in training can enhance the consistency of generated videos. 
We theoretically prove that Local Opt. achieves lower cumulative error than the Fewer-Frames model while maintaining baseline-level inference time.
Further analyzing the loss distribution across frames reveals that token generation difficulty decreases as the sequence progresses, and poor quality in initial frames adversely affects subsequent frames. 
To mitigate this, we prioritize window sampling in initial frames, significantly improving video quality.

To further enhance consistency in videos generated by Local Opt., we draw inspiration from \textit{Lipschitz continuity} and propose the Representation Continuity (ReCo) training method, as shown in Figure~\ref{fig:intro}. 
In addition to cross-entropy loss within each window, ReCo incorporates a representation continuity loss. 
We theoretically demonstrate that ReCo reduces generation errors and enhances both video quality and consistency compared to Local Opt. 
Experimental evaluations on class- and text-to-video generation datasets show that our approach outperforms existing autoregressive video generation methods, achieving twice the training speed of the baseline. 
Furthermore, the loss distribution of video tokens in our method is smoother than that of the baseline, and our approach matches the baseline in video consistency.

Our main contributions are as follows:
\begin{itemize}[leftmargin=*, itemsep=0pt]
\item We conduct empirical analysis to accelerate training for autoregressive video generation models. The analysis provides insights, including that the Fewer-Frames Model trains faster but produces videos with notable inconsistencies.
\item We propose the Local Opt. training method and show its advantages over the Fewer-Frames Model in both training and inference through extensive experiments. Additionally, we theoretically prove that Local Opt. reduces error accumulation compared to the Fewer-Frames Model.
\item We propose a Representation Continuity strategy to improve Local Opt. consistency while retaining training speed advantages. Both theory and experiments show that this strategy reduces error accumulation and enhances consistency, achieving better-than-baseline performance with half the training cost.
\end{itemize}

\section{Related Work}\label{app:related}
\subsection{Diffusion-based Video Generation}
Text-to-video generation has attracted considerable attention in recent years, fueled by advancements in deep generative models such as diffusion models and Transformers \citep{stablevideodiffusion,confiner,phenaki,modelscope}. Diffusion models have demonstrated strong capabilities in generating high-quality and temporally consistent videos. Approaches like VideoLCM~\citep{videolcm} and Stable Video Diffusion~\citep{stablevideodiffusion} leverage latent spaces and temporal consistency losses to ensure scalable and coherent video generation. Similarly, frameworks such as Show-1~\citep{show1}, VideoFactory~\citep{videofactory}, and Latte~\citep{latte} introduce techniques like hierarchical latent spaces, spatial-temporal attention mechanisms, and Transformer-based modeling to enhance video quality, computational efficiency, and flexibility in generating long sequences. ConFiner~\citep{confiner} further contributes by proposing a training-free framework that utilizes diffusion model experts for temporal control, eliminating the need for extensive retraining. Transformer-based architectures, on the other hand, excel at capturing long-range temporal dependencies. CogVideo~\citep{cogvideo} and Phenaki~\citep{phenaki} utilize Transformers pretrained on large datasets to achieve variable-length and cross-domain video generation, ensuring temporal consistency across diverse scenes. Latte~\citep{latte} combines Transformer-based temporal modeling with latent diffusion, effectively bridging these two paradigms. Some works explore user-centric strategies to enhance video generation. InstructVideo~\citep{instructvideo} integrates reinforcement learning with human feedback to better align outputs with textual instructions, particularly for ambiguous prompts. ModelScope~\citep{modelscope} focuses on usability and scalability, offering an open-source framework with a modular design for diverse text-to-video applications. Recent studies also investigate efficient and controllable diffusion generation, including locality-aware dynamic rescue for diffusion LLMs~\citep{wang2026ladr}, hierarchical compositional generation via reinforcement learning~\citep{yang2025hicogen}, decoupling inter- and intra-element conditions~\citep{yang2025dc}, stabilized diffusion Transformers through long-skip connections with spectral constraints~\citep{chen2025towards}, and self-rewarding large vision-language models for prompt optimization~\citep{yang2025self}.

\subsection{Autoregressive Video Generation}
Autoregressive video generation has gained traction due to its ability to capture long-range dependencies, particularly inspired by the success of large language models (LLMs) \citep{zhou2024visual,zhou2025weak,zhou2023thread}. Some works like \cite{VideoGPT,ren2025next} utilize a two-stage pipeline with VQ-VAE for video compression and Transformers for temporal modeling, achieving scalable and temporally consistent outputs. Simplifying this pipeline, \citet{NOVA2023} directly models pixel sequences with autoregressive Transformers, reducing architectural complexity while maintaining competitive quality.
Recent advancements focus on improving efficiency and extending temporal coherence. 
For example, Emu3 \citep{Emu3} leverages LLM-inspired architectures for efficient sequence modeling, while Loong \citep{Loong} employs a hierarchical autoregressive structure to generate minute-long videos. 
In the image domain, \citet{songbroad} proposes entropy-guided optimization for stable autoregressive synthesis, and \citet{zhoucondition} refines condition errors in autoregressive generation with diffusion loss.
Similarly, \citet{VRC} emphasizes compressing visual representations to reduce redundancy, improving scalability and computational efficiency. Similarly, \citet{DiCoDe2023} combines diffusion processes with deep token compression, producing high-quality latent representations that are sequenced using Transformers. Additionally, \citet{Multimodal2023} integrates diffusion processes with multimodal learning, ensuring temporal and semantic coherence across modalities.
Moreover, \citet{MovieGen2023}, which adopts a modular design to synergize foundational models for each modality, enables semantically coherent, content-rich video outputs adaptable to diverse tasks.
However, current research predominantly focuses on video modeling and inference speed improvements while neglecting the computational overhead during training.

\section{Preliminaries}
Autoregressive video generation formulates video synthesis as a token-by-token prediction using language models, e.g., GPT \citep{GPT2}. 
In this paradigm, videos are first tokenized via VQ-VAE, e.g., OmniTokenizer \citep{OmniTokenizer}, which discretizes video representations to facilitate autoregressive modeling.

\paragraph{Discrete Video Representations.}
VQ-VAE encodes each video frame into a continuous latent space and then quantizes it into discrete codes. Given a video $\mathbf{V} = \{ \mathbf{v}_1, \dots, \mathbf{v}_t \}$, frames are compressed both temporally and spatially. Temporal compression reduces the number of frames by a factor of $\alpha$, resulting in $\bar{t} = 1 + \frac{t - 1}{\alpha}$ frame blocks, with the first frame preserved. Spatial compression reduces the resolution of each frame. Specifically, each input video $\mathbf{V}$ is mapped to a sequence of quantized latent vectors by first encoding it into continuous representations $\mathbf{Z} = \{ \mathbf{z}_1, \dots, \mathbf{z}_{\bar{t}} \} = f_{\text{enc}}(\mathbf{V})$, and then quantizing each $\mathbf{z}_t$ to its nearest codebook entry from $\mathcal{E} = \{ \mathbf{e}_1, \dots, \mathbf{e}_M \}$:
\begin{align}
\mathbf{e}_t = \arg\min_{\mathbf{e}_i \in \mathcal{E}} \| \mathbf{z}_t - \mathbf{e}_i \|_2.
\end{align}
\paragraph{Autoregressive Video Generation.}
Given a condition $\mathbf{c}$, an autoregressive model generates each token $\mathbf{e}_t$ based on all previous tokens $\mathbf{e}_{<t}$ and the condition. The resulting token sequence $\mathbf{E} = (\mathbf{e}_1, \mathbf{e}_2, \dots, \mathbf{e}_{\bar{t}})$ is then decoded by a VQ-VAE decoder $f_{\text{dec}}$ to reconstruct the video $\mathbf{V}$:
\begin{align}
\mathbf{V} = f_{\text{dec}}(\mathbf{E}), P(\mathbf{E} \mid \mathbf{c}) = \prod_{t=1}^{\bar{t}} P(\mathbf{e}_t \mid \mathbf{e}_{<t}, \mathbf{c})
\end{align}

\section{Accelerating AR Video Modeling}
\subsection{Training on Fewer Frames}

\begin{figure}
    \centering
    \includegraphics[width=1\linewidth]{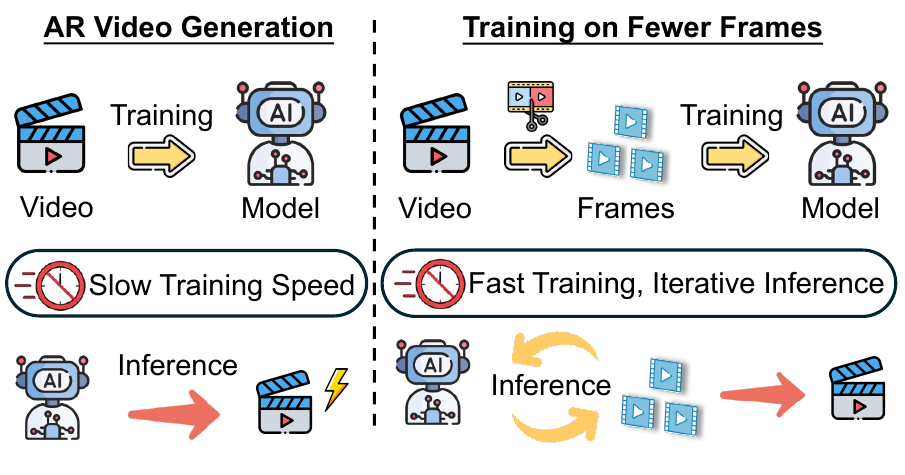}
    \vspace{-6mm}
    \caption{\small Comparison of AR video model training methods: ({\bf Left}) full-video training vs. ({\bf Right}) fewer-frame training.}
    \label{fig:fewer}
\end{figure}

\begin{table}[t]\small
\centering
\begin{tabular}{lcccc}
\toprule
\bf Method & \bf FFS & \bf SKY & \bf Train Speed ↑ \\
\midrule
Baseline & ~73.65 & ~89.09 & 0.84~~~~~~~~~~~~ \\
Fewer-Frames & 229.32 & 292.41 & 1.75 \textcolor{blue}{($\times$2.5)} \\
\bottomrule
\end{tabular}
\caption{\small 
Comparison of FVD performance across methods on the FFS and SKY datasets. 
`Train Speed'' is the training time for 24 videos on one A100 GPU (\textcolor{blue}{blue}: speedup factor).
}
\label{tab:ar_comparison}
\end{table}
Training autoregressive video models on fewer frames can significantly reduce the overall training time. 
To investigate whether models trained on shorter sequences can still generate videos with longer frame sequences, we train a Fewer-Frames model on 2-frame blocks, while the baseline is trained on 5-frame blocks (17 frames). 
In inference, the Fewer-Frames model employs an iterative approach. After each frame block is generated, it is re-input to the model to predict the subsequent frame block, and this process continues until five frame blocks are obtained.
More details can be found in Appendix~\ref{app:details}.
We evaluated the performance of both models on the FaceForensics (FFS \citep{FFS}) and SkyTimelapse (SKY \citep{SKY}) datasets, using the Fréchet Video Distance (FVD \citep{FVD}) as the evaluation metric.

\begin{figure}[t]
\centering
\includegraphics[width=0.485\linewidth]{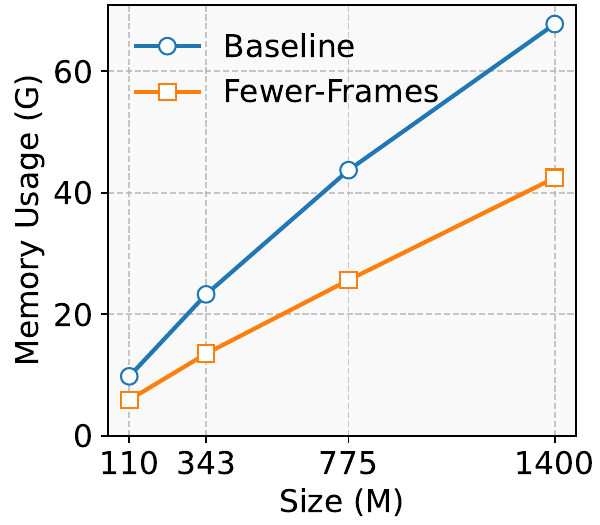}
\includegraphics[width=0.485\linewidth]{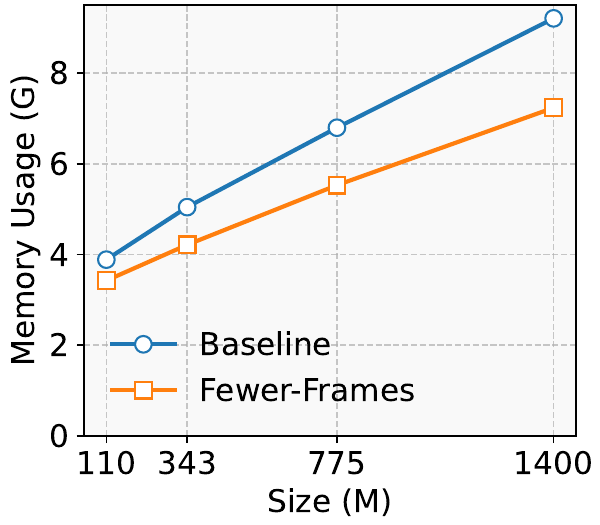}
\caption{\small 
Comparison of GPU memory usage for Baseline and Fewer-Frames with different model sizes during training ({\bf Left}) and inference ({\bf Right}), measured on a single A100 GPU with 2 videos.
}
\label{fig:ar_gpu}
\end{figure}

From Table \ref{tab:ar_comparison}, we show the performance comparison between the Baseline and Fewer-Frames methods on the FFS and SKY datasets. 
The Fewer-Frames method significantly accelerates training due to the reduced number of frames used during the training process. However, this benefit comes at the cost of lower video quality. The lower video quality is primarily attributed to the iterative approach, which requires repeatedly reloading the initial frame chunks during inference. It also impacts the model's ability to maintain consistency across the full video sequence. 
However, the Fewer-Frames shows a significant advantage in memory usage. In Figure \ref{fig:ar_gpu}, it uses considerably less GPU memory during both training and inference compared to the Baseline.

\subsection{Inconsistency on Fewer-Frames Model}
The primary drawback of the Fewer-Frames model emerges during inference, where its iterative, block-by-block generation process leads to a compounding of errors and causes temporal inconsistencies. To understand why, we analyze the difference in its autoregressive conditioning compared to the Baseline.

Let the ground-truth sequence of token blocks be $\mathbf{T} = (\mathbf{T}_1, \dots, \mathbf{T}_K)$, where each block $\mathbf{T}_k$ represents a segment of the video.

\paragraph{Baseline Model.} The Baseline model generates the entire video in a single, continuous pass. The prediction of any given block $\mathbf{T}_k$ is conditioned on all previously generated blocks, providing a global context:
\begin{align}
P(\mathbf{T}_k | \hat{\mathbf{T}}_{<k}) = P(\mathbf{T}_k | \hat{\mathbf{T}}_1, \dots, \hat{\mathbf{T}}_{k-1})
\end{align}
This long-range conditioning is crucial for maintaining temporal consistency throughout the sequence.

\paragraph{Fewer-Frames Model.} In contrast, the Fewer-Frames model is trained on isolated, short sequences. During inference, it generates the video block by block, where the prediction of block $\mathbf{T}_k$ is conditioned \textit{only} on the immediately preceding block $\hat{\mathbf{T}}_{k-1}$:
\begin{align}
P(\mathbf{T}_k | \hat{\mathbf{T}}_{k-1})
\end{align}
This limited context is the root of the problem. Any generation error in block $\hat{\mathbf{T}}_{k-1}$ is not merely a part of the history; it becomes the \textit{entire} basis for generating the subsequent block $\mathbf{T}_k$. This creates a \textit{cascading error} effect, where deviations quickly accumulate and cause the generated sequence to drift significantly from the true data distribution.

This analysis leads to the following proposition, which formalizes the expected error accumulation.

\begin{proposition}
\label{prop:error_accumulation}
Let $\hat{\mathbf{T}}_{Base}$ and $\hat{\mathbf{T}}_{FF}$ be the sequences generated by the Baseline and Fewer-Frames models, respectively. The expected deviation of the Fewer-Frames model from the true sequence $\mathbf{T}$ is greater than or equal to that of the Baseline model:
\begin{align}
\mathbb{E}[\|\mathbf{T} - \hat{\mathbf{T}}_{FF}\|] \geq \mathbb{E}[\|\mathbf{T} - \hat{\mathbf{T}}_{Base}\|]
\end{align}
(The formal proof can be found in Appendix~\ref{app:the1}).
\end{proposition}

This theoretical result explains the empirical findings in Table~\ref{tab:ar_comparison}, where the Fewer-Frames model yields a substantially higher FVD score. The accumulated errors manifest as tangible visual artifacts, such as temporal inconsistencies and object distortion, degrading the overall video quality.

\begin{figure}[t]
\centering
\includegraphics[width=\linewidth]{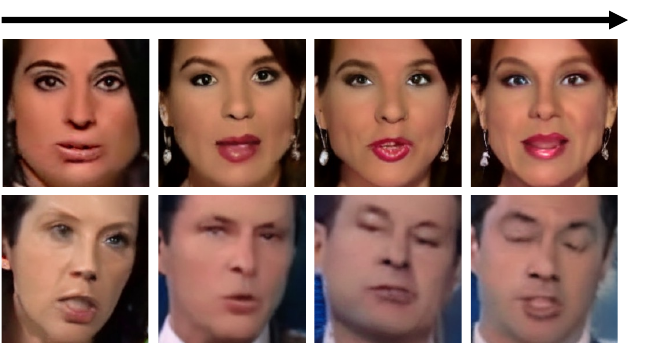}
\caption{\small Examples of video frames generated by the Fewer-Frames Model, showing frames 1, 5, 9, and 13 for two generated videos (top and bottom rows). As the frame number increases, error accumulation leads to noticeable inconsistencies in the appearance of the individuals within the video. Black arrow denotes video progression direction.}
\label{fig:incon_case}
\end{figure}

\paragraph{Empirical Analysis.}
As shown in Figure~\ref{fig:incon_case}, some videos generated by the Fewer-Frames model exhibit noticeable content inconsistencies (more cases in Appendix~\ref{appendix:cases}).
Specifically, as the video progresses, error accumulation becomes apparent, leading to pronounced inconsistencies in the appearance of individuals. 
To further analyze the consistency of the generated videos, we evaluate the Fewer-Frames model and the Baseline using PSNR and Optical Flow, as shown in Figure~\ref{fig:ar_psnr}. 
From Figure~\ref{fig:ar_psnr} (Left), the Fewer-Frames model produces lower PSNR values across varying source-target frame intervals compared to the Baseline. 
This highlights the reduced frame quality in videos generated by the Fewer-Frames model. 
Moreover, Figure~\ref{fig:ar_psnr} (Right) reveals that the Fewer-Frames model exhibits higher Optical Flow values, especially as the source-target frame interval increases. 
This indicates a greater inconsistency in motion across frames, aligning with the observation in Figure~\ref{fig:incon_case}.

\subsection{Training with Local Optimization}

To mitigate the error accumulation of the Fewer-Frames method while retaining its efficiency, we introduce \textit{Local Optimization} (Local-Opt.). It isolates the training objective to small, local windows of the video sequence, as shown in Figure~\ref{fig:local}.

Formally, given a full token sequence $\mathbf{E} = (\mathbf{e}_1, \dots, \mathbf{e}_N)$, each training step begins by sampling a random starting index $s$ to define a window of length $W$, denoted as $\mathbf{E}_{\mathcal{W}} = (\mathbf{e}_s, \dots, \mathbf{e}_{s+W-1})$. The training objective is to minimize the negative log-likelihood of the tokens \textit{only within this window}, conditioned on the preceding ground-truth sequence $\mathbf{E}_{<s} = (\mathbf{e}_1, \dots, \mathbf{e}_{s-1})$:
\begin{align}
\mathcal{L}_{\mathcal{W}}(\theta) = - \sum_{i=s}^{s+W-1} \log P(\mathbf{e}_i | \mathbf{E}_{<i}; \theta)
\label{eq:local_loss}
\end{align}
A crucial implementation detail is that the context $\mathbf{E}_{<s}$ is treated as a frozen constant. No gradients are propagated back through the representations of tokens outside the optimization window $\mathcal{W}$. This effectively uses a stop-gradient operation on the context, focusing the entire optimization effort on learning local transitions correctly.

\begin{figure}[t]
\centering
\includegraphics[width=0.485\linewidth]{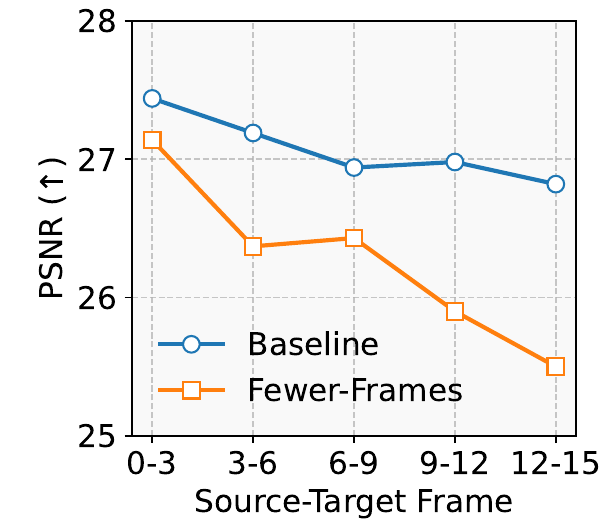}
\hfill
\includegraphics[width=0.485\linewidth]{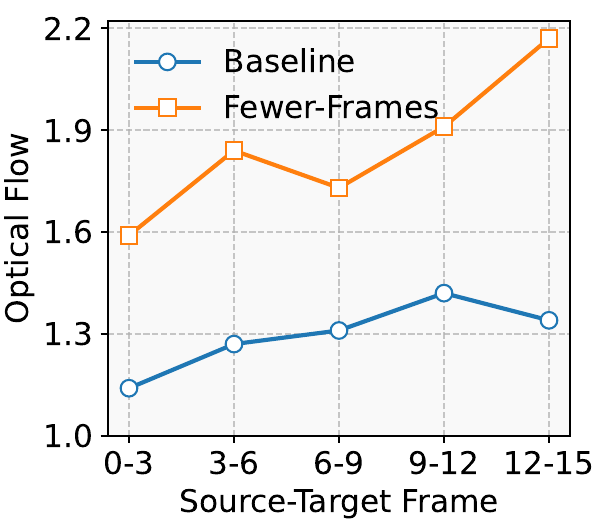}
\caption{\small Comparison of the Fewer-Frames Model with the Baseline using PSNR ({\bf Left}) and Optical Flow ({\bf Right}), measured across varying source-target frames. The left plot shows the PSNR values (higher is better), i.e., the quality of generated frames, while the right plot presents the Optical Flow values, i.e., the consistency of motion across frames.}
\label{fig:ar_psnr}
\end{figure}

This formulation provides two primary benefits. First, by always conditioning on an error-free ground-truth context, the model learns accurate predictions without the exposure bias inherent to iterative generation. Second, using a stride $S < W$ creates overlapping windows, meaning tokens are optimized multiple times under different contexts. This multi-context optimization enhances temporal consistency by forcing the model to learn more robust representations.
Importantly, Local-Opt. is a \textit{training-only} strategy. The inference procedure remains the standard, full-sequence autoregressive generation of the Baseline. This allows us to achieve significant training acceleration without compromising inference speed.

\begin{figure}[t]
    \centering
    \includegraphics[width=1\linewidth]{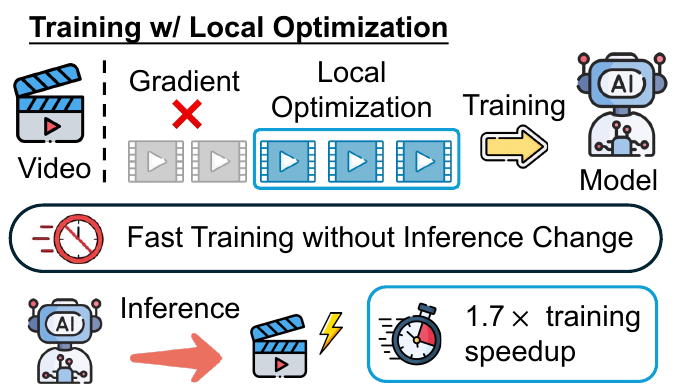}
    \caption{\small The Local Optimization (Local-Opt.) training strategy. During training, a window of frames is randomly selected for optimization (blue), while preceding frames serve as a frozen context (gray), preventing gradient flow. }
    \label{fig:local}
\end{figure}

\begin{proposition}\label{prop:local_opt_error}
Given the same history, the Local-Opt. model is explicitly trained to minimize prediction error within a window $\mathcal{W}$ via Eq.~\ref{eq:local_loss}. Consequently, its expected squared error within that window is lower than or equal to that of the Fewer-Frames model:
\begin{align}
\!\!\!\!\mathbb{E}[\|\mathbf{E}_{\mathcal{W}} - \hat{\mathbf{E}}_{LO,\mathcal{W}}\|^2] \leq \mathbb{E}[\|\mathbf{E}_{\mathcal{W}} - \hat{\mathbf{E}}_{FF,\mathcal{W}}\|^2]\!\!
\end{align}
(The formal proof and an explanation of the benefits of overlapping windows are provided in Appendix~\ref{app:the2}).
\end{proposition}

\begin{figure}[t]
\centering
\includegraphics[width=0.34\linewidth]{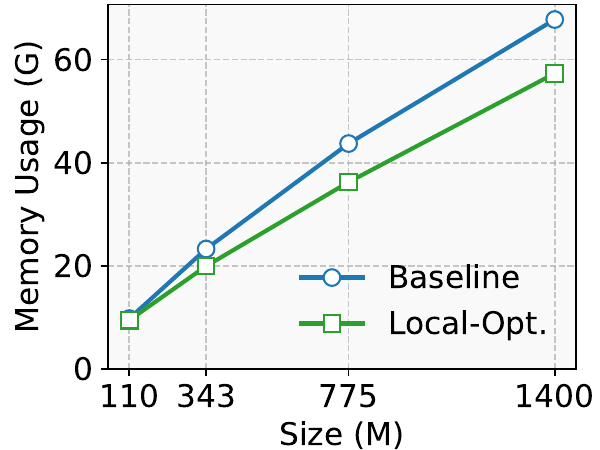}\hfill
\includegraphics[width=0.31\linewidth]{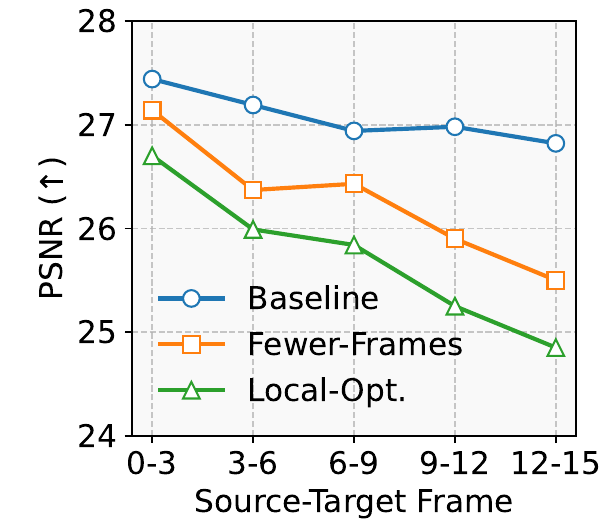}\hfill
\includegraphics[width=0.31\linewidth]{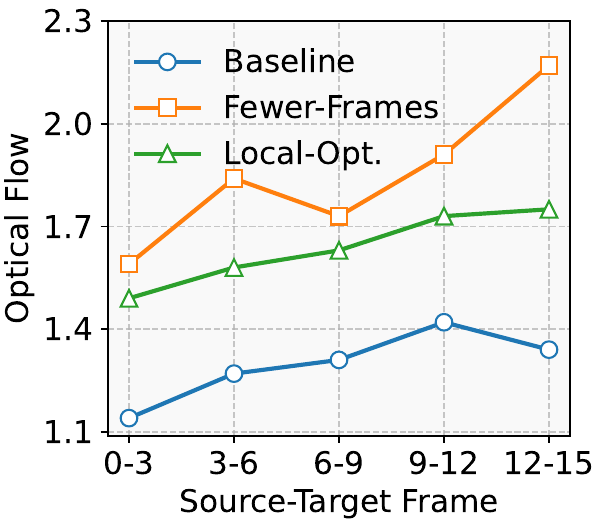}
\caption{\small 
({\bf Left}) Comparison of GPU memory usage during training for Baseline and Local-Opt. methods with different model sizes. 
The memory usage is measured on a single A100 GPU with a batch size of 2.
Comparison of the Local-Opt. and Fewer-Frames Model with the Baseline using PSNR ({\bf Middle}) and Optical Flow ({\bf Right}), measured across varying source-target frames. 
}
\label{fig:local_gpu}
\label{fig:local_psnr}
\end{figure}

\begin{table}[t]\small
\centering
\begin{tabular}{lcccc}
\toprule
\bf Method & \bf FFS & \bf SKY & \bf Train Speed ↑ \\
\midrule
Baseline & ~73.65 & ~89.09 &  0.84~~~~~~~~~~~~ \\
Fewer-Frames & 229.32 & 292.41 & 2.10 \textcolor{blue}{($\times$2.5)} \\
Local-Opt. & 190.46 & 256.94 & 1.47 \textcolor{blue}{($\times$1.7)} \\
\bottomrule
\end{tabular}
\caption{\small Performance and training speed comparison of methods on FFS and SKY datasets.}
\label{tab:lo_comparison}
\end{table}

\paragraph{Empirical Analysis.}  
Table~\ref{tab:lo_comparison} demonstrates that Local-Opt. achieves a 1.7 $\times$ training speedup over the Baseline. Additionally, Local-Opt. balances computational efficiency with improved accuracy, achieving higher FFS and SKY scores than the Fewer-Frames. 
Figure~\ref{fig:local_gpu} (Left) shows Local-Opt.'s memory efficiency, reducing GPU memory usage compared to the Baseline at larger model sizes.  
In Figure~\ref{fig:local_psnr} (Middle), Local-Opt. achieves slightly lower PSNR compared to Fewer-Frames. However, optical flow results from Figure~\ref{fig:local_psnr} (Right) show consistent performance across varying source-target frames, demonstrating Local-Opt.'s ability to maintain temporal consistency effectively.
By optimization on overlapping windows, Local-Opt. can mitigate the error propagation characteristic of the Fewer-Frames approach. The training algorithm is in Appendix~\ref{appendix:local_opt_algorithm}.

\begin{figure}[ht]
\centering
\includegraphics[width=\linewidth]{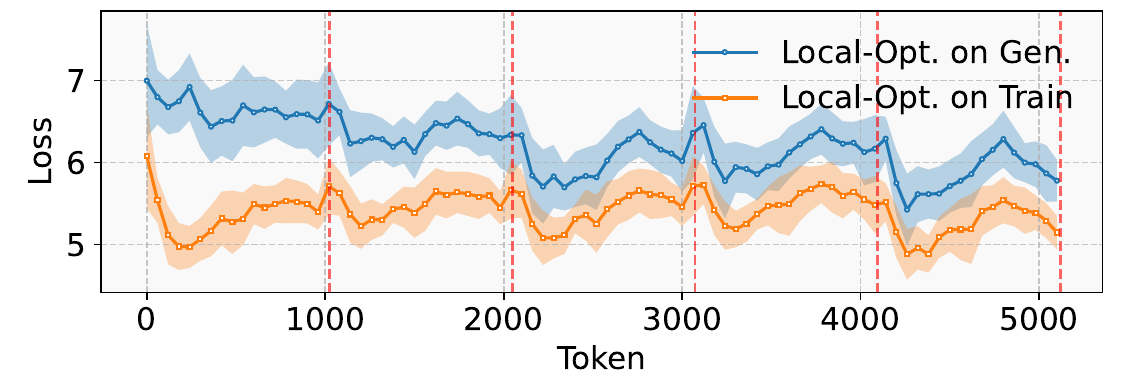}
\includegraphics[width=\linewidth]{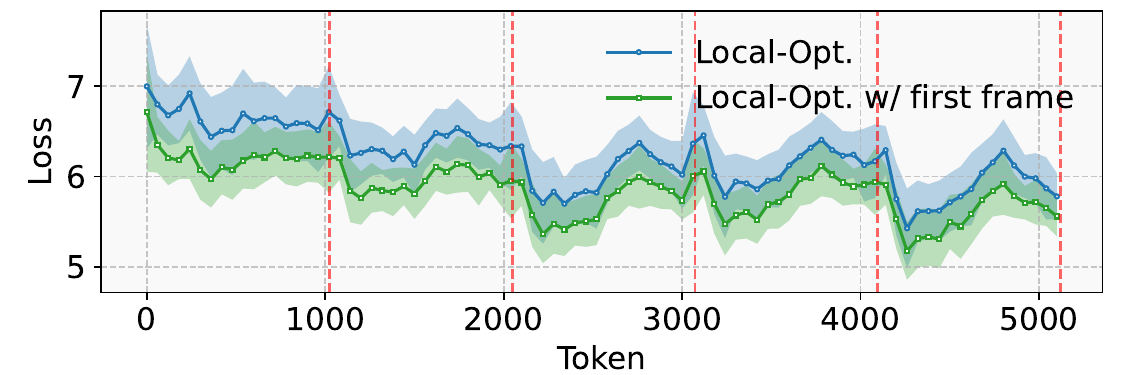}
\caption{\small Loss distribution of different optimization strategies on training and generated samples. ({\bf Top}) Comparison of loss distributions between ``Local-Opt.'' on its generated samples (``Local-Opt. on Gen.'') and its training samples (``Local-Opt. on Train''). ({\bf Bottom}) Comparison of loss distributions for ``Local-Opt.'' and ``Local-Opt. w/ first frame'', where the first frame of the video is provided as ground truth, on generated samples. The red dashed lines indicate frame blocks corresponding to the spatiotemporal compression of VQVAE. }
\label{fig:source_gen}
\end{figure}

\subsection{Imbalance on Local-Optimization Model}
\begin{figure}[t]
\centering
\includegraphics[width=\linewidth]{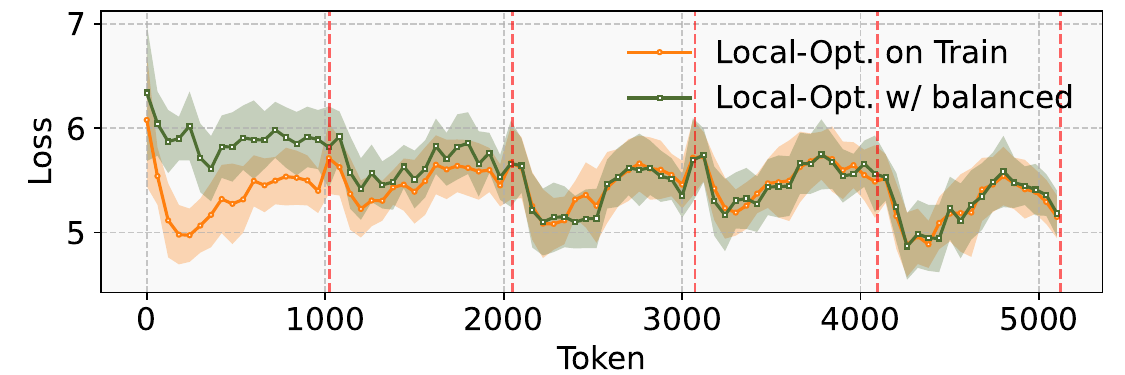}
\caption{\small Loss distribution comparison between ``Local-Opt. w/ balanced (first frame balanced)'' on generated samples, and "Local-Opt." on training samples. }
\label{fig:gpt_balanced}
\end{figure}
We observed a significant discrepancy in the loss distributions between the training data and the generated data of the Local-Opt. model. 
Specifically, as shown in Figure~\ref{fig:source_gen} (Top), the loss on the generated samples (``Local-Opt. on Gen.'') tends to be higher and more uneven compared to the loss on the training samples (``Local-Opt. on Train''). 
This imbalance reveals a limitation in the Local-Opt. model's ability to generalize effectively across training and generation phases.  
To mitigate this issue, we introduced a variant, i.e., ``Local-Opt. w/ first frame'', in which the first frame of the generated video is provided from ground truth during inference. 
As shown in Figure~\ref{fig:source_gen} (Bottom), this approach significantly reduces the loss across most token indices in the generated data, achieving a more stable and lower distribution.

\begin{table}[t]\small
\centering
\setlength{\tabcolsep}{3.5pt}
\resizebox{\linewidth}{!}{
\begin{tabular}{lccc}
\toprule
\textbf{Method}                     & \textbf{FFS}  & \textbf{SKY}    & \textbf{Train Speed ↑} \\ 
\midrule
Baseline                            & ~73.65         & ~89.09           & 0.84~~~~~~~~~~~~ \\ 
Local-Opt.                          & 190.46        & 256.94          & 1.47 \textcolor{blue}{($\times$1.7)} \\ 
Local-Opt. (\textit{w/ first frame})            & 134.73        & 186.63          & 1.47 \textcolor{blue}{($\times$1.7)} \\ 
Local-Opt. (\textit{w/ balanced})               & 127.11        & 179.84          & 1.68 \textcolor{blue}{($\times$2.0)} \\ 
\bottomrule
\end{tabular}}
\caption{\small Performance and training speed comparison of Local-Opt. and its variants with the Baseline on the FFS and SKY datasets.}
\label{tab:balance_comparison}
\end{table}

The FVD scores in Table~\ref{tab:balance_comparison} under ``Local-Opt. (first frame augmented)'' confirm this improvement.
Therefore, we adjusted the sampling strategy for the training process. 
Specifically, we increased the sampling proportion of the window containing the first frame to 0.5, creating a balanced training variant termed ``Local-Opt. w/ balanced''. 
As shown in Figure~\ref{fig:gpt_balanced}, this adjustment resulted in a further alignment of the loss distributions between the training and generated data. 
While the loss for the first frame remains slightly higher, the losses for subsequent frames exhibit significantly improved consistency. This suggests that the rebalancing strategy effectively reduces the generalization gap. The FVD and training speed for ``Local-Opt. w/ balanced'', as shown in Table~\ref{tab:balance_comparison}, shows a significant performance gain compared to both the baseline and the original Local-Opt. model.

\section{Representation Continuity}

While Local-Opt. accelerates training, its focus on independent windows can permit abrupt transitions in the learned representation space, limiting overall video consistency. To address this, we propose \textit{Representation Continuity} (ReCo), a regularization strategy designed to enforce temporal smoothness on the model's internal dynamics. Our approach is inspired by the principle of \textit{Lipschitz continuity}.

\subsection{Autoregressive Models as Dynamical Systems}

An autoregressive model can be viewed as a discrete-time dynamical system. Let $\mathbf{h}_{t-1}$ be the hidden state representation summarizing the history up to time $t-1$, and $\mathbf{e}_t$ be the embedding of the current input token. The model computes the next hidden state $\mathbf{h}_t$ via a state transition function $g$:
\begin{align}
\mathbf{h}_t = g(\mathbf{h}_{t-1}, \mathbf{e}_t; \theta)
\end{align}
The function $g$ is a complex, non-linear function parameterized by the model's weights $\theta$. The stability and smoothness of the generated sequence are fundamentally governed by the properties of this transition function.

For video generation, we desire a system where the sequence of hidden states $(\mathbf{h}_1, \mathbf{h}_2, \dots)$ evolves smoothly over time. A desirable property for $g$ is to be Lipschitz continuous with respect to its recurrent input $\mathbf{h}_{t-1}$.

\begin{definition}[Lipschitz Continuity]
\label{def:lipschitz}
A function $g(\mathbf{h}, \cdot)$ is L-Lipschitz continuous with respect to $\mathbf{h}$ if there exists a constant $L \ge 0$ such that for any two states $\mathbf{h}_a, \mathbf{h}_b$, the following holds:
\begin{align}
\| g(\mathbf{h}_a, \mathbf{e}) - g(\mathbf{h}_b, \mathbf{e}) \|_2 \leq L \cdot \| \mathbf{h}_a - \mathbf{h}_b \|_2
\end{align}
A small Lipschitz constant $L$ ensures that small perturbations in the hidden state do not get amplified, leading to a more stable system.
\end{definition}

\subsection{The ReCo Loss as a Lipschitz Regularizer}
Directly computing or constraining the global Lipschitz constant $L$ of a deep neural network is computationally intractable. Instead, we propose a practical proxy: a \textit{representation continuity loss} ($\mathcal{L}_{ReCo}$) that encourages local smoothness by penalizing large variations in the hidden state across consecutive time steps.

For a window $\mathcal{W}$, let $\mathbf{H}_{\mathcal{W}} = (\mathbf{h}_s, \dots, \mathbf{h}_{s+W-1})$ be the sequence of hidden representations. We define $\mathcal{L}_{ReCo}$ as the mean squared distance between adjacent hidden states:
\begin{align}
\label{eq:reco_loss}
\!\!\!\mathcal{L}_{ReCo}(\mathbf{H}_{\mathcal{W}}) = \frac{1}{W-1} \!\!\sum_{i=s}^{s+W-2} \!\!\| \mathbf{h}_{i+1} - \mathbf{h}_{i} \|^2_2
\end{align}
By minimizing this term, we implicitly encourage the transition function $g$ to produce outputs $\mathbf{h}_{i+1}$ that are close to its inputs $\mathbf{h}_i$, which is analogous to encouraging a small local Lipschitz constant. The total loss for the window is a weighted sum of the autoregressive cross-entropy loss $\mathcal{L}_{CE}$ and our continuity regularizer:
\begin{align}
\label{eq:total_loss}
\!\!\!\mathcal{L}_{Total} = \mathcal{L}_{CE}(\mathbf{E}_{\mathcal{W}}, \hat{\mathbf{E}}_{\mathcal{W}}) + \lambda \!\cdot\! \mathcal{L}_{ReCo}(\mathbf{H}_{\mathcal{W}})\!
\end{align}
where $\lambda$ balances prediction accuracy with representation smoothness.

\subsection{Theoretical Justification}
The primary benefit of enforcing representation continuity is the reduction of error accumulation during inference. Consider the generation process where the model is fed its own, potentially erroneous, predictions. Let $\mathbf{h}_t$ be the state generated from a ground-truth history, and $\hat{\mathbf{h}}_t$ be the state from a generated history. The error at step $t$ is $\boldsymbol{\epsilon}_t = \hat{\mathbf{h}}_t - \mathbf{h}_t$. The error propagates as follows:
\begin{align}
\boldsymbol{\epsilon}_{t+1} = g(\hat{\mathbf{h}}_t, \hat{\mathbf{e}}_{t+1}) - g(\mathbf{h}_t, \mathbf{e}_{t+1})
\end{align}
If $g$ has a small Lipschitz constant $L$, the growth of the error is bounded:
\begin{align}
\| \boldsymbol{\epsilon}_{t+1} \| \le L \cdot \| \boldsymbol{\epsilon}_t \| + \delta_t
\end{align}
where $\delta_t$ represents the new error introduced by predicting token $\hat{\mathbf{e}}_{t+1}$. By regularizing the model with $\mathcal{L}_{ReCo}$, we effectively encourage a smaller $L$, thus suppressing the exponential amplification of errors and improving the stability of long-sequence generation. Smoother latent dynamics also translate to more consistent visual outputs from the decoder, which can be verified by metrics like Optical Flow.

\begin{proposition}\label{prop:reco_improves_consistency}
By regularizing the state transition function via the Representation Continuity loss (Eq.~\ref{eq:reco_loss}), the ReCo model learns smoother latent dynamics. This theoretically bounds error propagation during inference and is expected to yield generated videos with higher temporal consistency compared to the standard Local-Opt. model.
\end{proposition}

\section{Experiments}
\subsection{Experimental Setting}\label{sec:setting}
In this study, we evaluate our model using four video generation datasets: FFS \citep{FFS}, SKY \citep{SKY}, UCF101 (UCF \citep{UCF}), and Taichi-HD (Taichi \citep{Taichi}). The model performance is reported using FVD \citep{FVD}. We train and evaluate the models on 17-frame, 256$\times$256 resolution videos. We utilize OmniTokenizer \citep{OmniTokenizer} to transform the videos into tokens. For both the baseline and our proposed models, we employ two different parameter sizes, 110M and 343M, denoted as Baseline$^\star$ and $\ours$$^\star$, respectively. The training process for both models spans 300 epochs, with a learning rate of $1 \times 10^{-4}$. $\gamma$ and $\lambda$ are 0.01 and 0.1, respectively. The batch sizes are set to 96 for the 110M parameter size and 40 for the 343M parameter size. All experiments are conducted on a cluster of four NVIDIA A100 GPUs. To provide a comprehensive comparison, we evaluate our models against three types of video generation models: GAN-based, Diffusion-based, and Autoregressive.

\subsection{Results and Discussion}

As shown in Table~\ref{tab:performance_comparison}, we compare our method with state-of-the-art video generation models, including GAN-, diffusion-, and autoregressive-based approaches. Our models consistently demonstrate the effectiveness of the proposed training strategy.
Both the base model ($\ours$) and its larger variant ($\ours^\bigstar$) outperform their respective baselines across all four datasets, confirming the benefits of local optimization and representation continuity. In particular, $\ours^\bigstar$ achieves better FVD scores on multiple datasets, scoring 42.5 on FFS, 58.8 on SKY, and 98.3 on Taichi.
Moreover, our approach is complementary to architectural advancements. When combined with the LARP tokenizer~\cite{wang2025larp}, the enhanced model ($\ours^\spadesuit$) achieves new best results on FFS (46.2) and UCF (56.1), surpassing the original LARP model. These results demonstrate that our training strategy is both effective and scalable, providing a strong and versatile foundation for autoregressive video generation.

\subsection{Text-to-Video Generation}
\label{sec:text2video}
\begin{table}[t]\small
\centering
\setlength{\tabcolsep}{2pt}
\resizebox{\linewidth}{!}{
\begin{tabular}{lcccc}
\toprule
\bf Method        & \bf FFS   & \bf SKY   & \bf UCF   & \bf Taichi \\
\hline\midrule
\multicolumn{5}{c}{\textit{GAN-based Video Generation Model}} \\\midrule
MoCoGAN \citep{MoCoGAN}      & 124.7  & 206.6  & 2886.9    & -         \\
MoCoGAN-HD \citep{MoCoGAN}    & 111.8  & 164.1  & 1729.6    & 128.1     \\
DIGAN \citep{DIGAN}        & ~~62.5 & ~~83.1 & 1630.2    & 156.7     \\
\hline\midrule
\multicolumn{5}{c}{\textit{Diffusion-based Video Generation Model}} \\\midrule
PVDM \citep{PVDM}         & 355.9  & ~~75.5 & 1141.9    & 540.2     \\
LVDM  \citep{LVDM}        & -      & ~~95.2 & ~~372.0   & ~~99.0    \\
\hline\midrule
\multicolumn{5}{c}{\textit{Autoregressive Video Generation Model}} \\\midrule
VideoGPT \citep{VideoGPT}     & 185.9  & 222.7  & 2880.6    & -         \\
Baseline \citep{Loong} & ~~73.7 & ~~89.1 & ~~630.8   & 115.5    \\
Baseline$^\bigstar$ \citep{Loong}  & ~~46.1  & ~~62.7 & ~~254.5 & 105.5   \\
\rowcolor{cyan!15}$\ours$ (Ours)      & ~~72.6 & ~~87.5 & ~~590.3   & 104.3    \\
\rowcolor{cyan!15}$\ours$$^\bigstar$ (Ours)      & \bf ~~42.5 & \bf ~~58.8 & \bf ~~251.4   & \bf ~~98.3   \\
\midrule
LARP \cite{wang2025larp} & ~~62.6 & ~~70.4 & ~~57.0 & 119.5 \\
\rowcolor{cyan!15}$\ours$$^\spadesuit$ (Ours) & \bf ~~46.2 & \bf ~~61.7 & \bf ~~56.1& \bf 104.6 \\
\bottomrule
\end{tabular}}
\caption{\small 
Performance comparison of video generation models across different datasets. $\bigstar$ denotes models with larger parameter sizes (770M). $\spadesuit$ denotes video model trained with LARP tokenizer~\cite{wang2025larp}.
}
\label{tab:performance_comparison}
\end{table}
\begin{table}[t]\small
\centering
\caption{Text-to-video generation results on MSR-VTT (zero-shot).}
\label{tab:text2video}
\setlength{\tabcolsep}{2.5pt}
\resizebox{\linewidth}{!}{
\begin{tabular}{lcccc}
\toprule
Method & Params & CLIP $\uparrow$ & FVD $\downarrow$ & Training Cost \\
\midrule
Loong \citep{Loong} & 7B & 0.2903 & 274 & 1.0 \\
\rowcolor{cyan!15}ReCo (Ours) & 7B & \textbf{0.3056} & \textbf{212} & 1.0 \\
\rowcolor{cyan!15}ReCo* (Ours) & 7B & 0.2911 & 267 & \textbf{0.5} \\
\bottomrule
\end{tabular}}
\end{table}
To further demonstrate the relevance of our method to NLP and multimodal generation, we evaluate ReCo on a text-to-video generation task.
We train the model on the Vimeo dataset containing approximately 300K video-text pairs and conduct zero-shot evaluation on MSR-VTT.
All methods are based on the Loong architecture with 7B parameters.
We report CLIP Score to measure text--video semantic alignment and FVD to assess video quality.
ReCo consistently improves both semantic alignment and video quality over the baseline.
Notably, ReCo* achieves competitive performance while using only 50\% of the training cost, demonstrating that our approach scales favorably to large multimodal models and significantly improves training efficiency without sacrificing generation quality.

\begin{figure}[t]
\centering
\includegraphics[width=\linewidth]{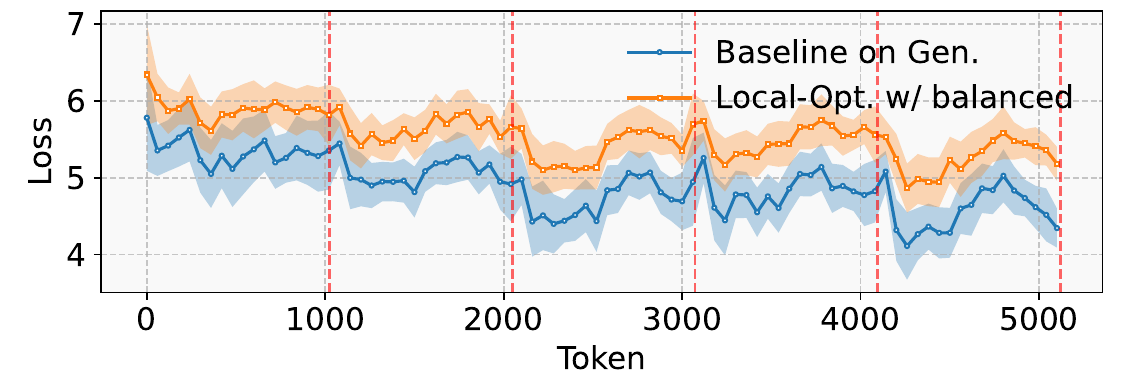}
\includegraphics[width=\linewidth]{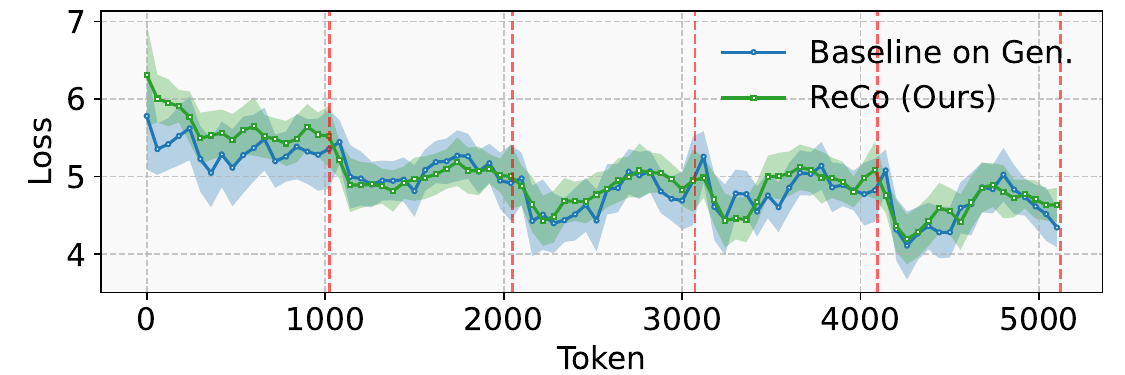}
\caption{\small Loss distribution comparison on generated samples. (Top) ``Baseline'' vs. ``Local-Opt. w/ balanced.'' (Bottom) ``Baseline'' vs. ``Ours''.}
\label{fig:gpt_ours}
\end{figure}

\paragraph{Loss Distribution Analysis.}
As shown in Figure~\ref{fig:gpt_ours}, we compare the loss distributions of generated samples across different methods. 
The top figure shows that the loss distribution of ``Local-Opt. w/ balanced'' consistently exceeds that of the ``Baseline'', indicating a persistent performance gap. 
In contrast, the bottom figure shows that ``ReCo (Ours)'' closely matches the ``Baseline'', except for a slightly higher loss at the first frame. 
Thereafter, ReCo demonstrates stable loss values with reduced fluctuation compared to the ``Baseline'', highlighting the robustness and stability of our approach.
Importantly, the incorporation of Representation Continuity (ReCo) into the ``Local-Opt. w/ balanced'' framework effectively mitigates the performance gap. It demonstrates that ReCo enhances the optimization process by promoting smoother transitions and more consistent representations across frames.

\begin{figure}[t]
\centering
\includegraphics[width=0.485\linewidth]{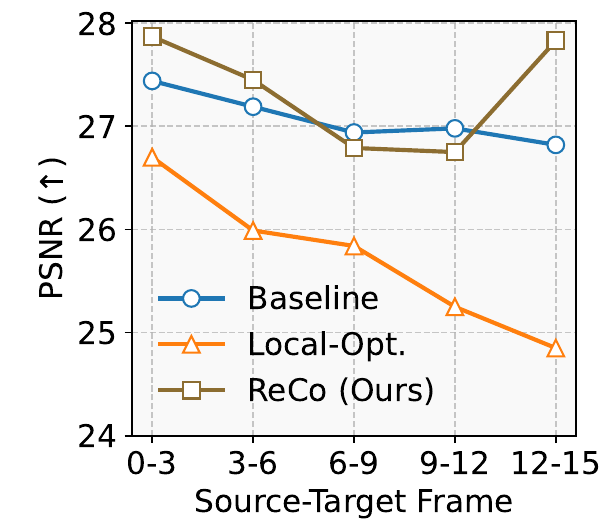}
\includegraphics[width=0.485\linewidth]{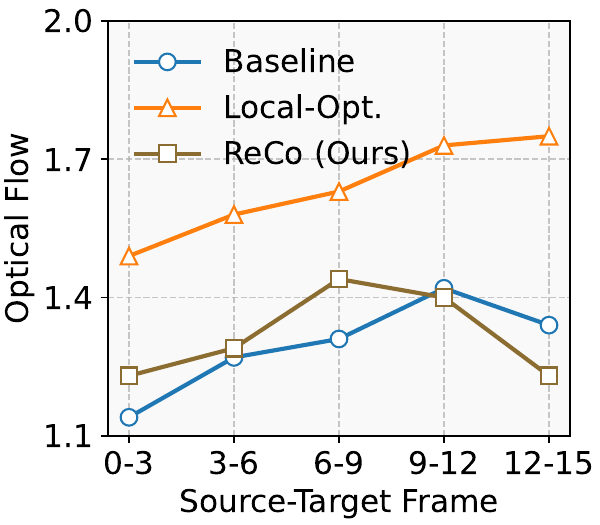}
\caption{\small 
Comparison of the Local-Opt. and Ours with the Baseline using PSNR ({\bf Left}) and Optical Flow ({\bf Right}), measured across varying source-target frames. }
\label{fig:ours_psnr}
\end{figure}

\paragraph{Video Consistency Analysis.}
We evaluate video consistency of our approach using PSNR and optical flow metrics (Figure~\ref{fig:ours_psnr}). ``ReCo (Ours)'' matches the  ``Baseline'' in PSNR, indicating similar perceptual quality, while surpassing ``Local-Opt.'' in optical flow, reflecting more realistic motion and frame coherence. Additionally, ReCo's training speed is on par with ``Local-Opt.'' and approximately 2 times faster than ``Baseline'', demonstrating both effectiveness and efficiency.

\section{Conclusion}
In this work, we explore methods to accelerate the training of autoregressive models for video generation. 
Through our empirical analysis, we found that while training with fewer frames can speed up the process, it leads to significant inconsistencies in the generated videos. 
To alleviate this, we propose the Local Optimization (Local Opt.) method, which not only mitigates error accumulation. 
Furthermore, by incorporating Representation Continuity (ReCo) training, our method can enhance video quality and consistency of  Local Opt., achieving substantial improvements in both training speed and overall performance. The consistency between our experimental findings and theoretical analysis further validates the effectiveness of our approach.

\section*{Limitations}\label{app:limitation}
This study mainly emphasizes algorithmic innovations rather than experiments on commercial-scale large language models. We believe, however, that the proposed methods are broadly applicable and intend to explore their potential at larger scales in future work.

\bibliography{ref}

@article{Emu3,
  author       = {Xinlong Wang and
                  Xiaosong Zhang and
                  Zhengxiong Luo and
                  Quan Sun and
                  Yufeng Cui and
                  Jinsheng Wang and
                  Fan Zhang and
                  Yueze Wang and
                  Zhen Li and
                  Qiying Yu and
                  Yingli Zhao and
                  Yulong Ao and
                  Xuebin Min and
                  Tao Li and
                  Boya Wu and
                  Bo Zhao and
                  Bowen Zhang and
                  Liangdong Wang and
                  Guang Liu and
                  Zheqi He and
                  Xi Yang and
                  Jingjing Liu and
                  Yonghua Lin and
                  Tiejun Huang and
                  Zhongyuan Wang},
  title        = {Emu3: Next-Token Prediction is All You Need},
  journal      = {CoRR},
  volume       = {abs/2409.18869},
  year         = {2024},
  url          = {https://doi.org/10.48550/arXiv.2409.18869},
  doi          = {10.48550/ARXIV.2409.18869},
  eprinttype    = {arXiv},
  eprint       = {2409.18869},
  bibsource    = {dblp computer science bibliography, https://dblp.org}
}

@article{MAR,
  author       = {Tianhong Li and
                  Yonglong Tian and
                  He Li and
                  Mingyang Deng and
                  Kaiming He},
  title        = {Autoregressive Image Generation without Vector Quantization},
  journal      = {CoRR},
  volume       = {abs/2406.11838},
  year         = {2024},
  url          = {https://doi.org/10.48550/arXiv.2406.11838},
  doi          = {10.48550/ARXIV.2406.11838},
  eprinttype    = {arXiv},
  eprint       = {2406.11838},
  bibsource    = {dblp computer science bibliography, https://dblp.org}
}

@article{LlamaGen,
  author       = {Peize Sun and
                  Yi Jiang and
                  Shoufa Chen and
                  Shilong Zhang and
                  Bingyue Peng and
                  Ping Luo and
                  Zehuan Yuan},
  title        = {Autoregressive Model Beats Diffusion: Llama for Scalable Image Generation},
  journal      = {CoRR},
  volume       = {abs/2406.06525},
  year         = {2024},
  url          = {https://doi.org/10.48550/arXiv.2406.06525},
  doi          = {10.48550/ARXIV.2406.06525},
  eprinttype    = {arXiv},
  eprint       = {2406.06525},
  bibsource    = {dblp computer science bibliography, https://dblp.org}
}

@misc{VRC,
  title={Less is more: Vision representation compression for efficient video generation with large language models},
  author={Zhou, Yucheng and Zhang, Jihai and Chen, Guanjie and Shen, Jianbing and Cheng, Yu},
  year={2024},
  publisher={OpenReview}
}

@article{Skip,
  author       = {Guanjie Chen and
                  Xinyu Zhao and
                  Yucheng Zhou and
                  Tianlong Chen and
                  Yu Cheng},
  title        = {Accelerating Vision Diffusion Transformers with Skip Branches},
  journal      = {CoRR},
  volume       = {abs/2411.17616},
  year         = {2024},
  url          = {https://doi.org/10.48550/arXiv.2411.17616},
  doi          = {10.48550/ARXIV.2411.17616},
  eprinttype    = {arXiv},
  eprint       = {2411.17616},
  bibsource    = {dblp computer science bibliography, https://dblp.org}
}

@article{OpenSora,
  title={Open-sora: Democratizing efficient video production for all},
  author={Zheng, Zangwei and Peng, Xiangyu and Yang, Tianji and Shen, Chenhui and Li, Shenggui and Liu, Hongxin and Zhou, Yukun and Li, Tianyi and You, Yang},
  journal={arXiv preprint arXiv:2412.20404},
  year={2024}
}

@article{Chameleon,
  author       = {Chameleon Team},
  title        = {Chameleon: Mixed-Modal Early-Fusion Foundation Models},
  journal      = {CoRR},
  volume       = {abs/2405.09818},
  year         = {2024},
  url          = {https://doi.org/10.48550/arXiv.2405.09818},
  doi          = {10.48550/ARXIV.2405.09818},
  eprinttype    = {arXiv},
  eprint       = {2405.09818},
  bibsource    = {dblp computer science bibliography, https://dblp.org}
}

@article{GPT2,
  title={Language models are unsupervised multitask learners},
  author={Radford, Alec and Wu, Jeffrey and Child, Rewon and Luan, David and Amodei, Dario and Sutskever, Ilya and others},
  journal={OpenAI blog},
  volume={1},
  number={8},
  pages={9},
  year={2019}
}

@article{FFS,
  author       = {Andreas R{\"{o}}ssler and
                  Davide Cozzolino and
                  Luisa Verdoliva and
                  Christian Riess and
                  Justus Thies and
                  Matthias Nie{\ss}ner},
  title        = {FaceForensics: {A} Large-scale Video Dataset for Forgery Detection
                  in Human Faces},
  journal      = {CoRR},
  volume       = {abs/1803.09179},
  year         = {2018},
  url          = {http://arxiv.org/abs/1803.09179},
  eprinttype    = {arXiv},
  eprint       = {1803.09179},
  bibsource    = {dblp computer science bibliography, https://dblp.org}
}

@article{OmniTokenizer,
  author       = {Junke Wang and
                  Yi Jiang and
                  Zehuan Yuan and
                  Bingyue Peng and
                  Zuxuan Wu and
                  Yu{-}Gang Jiang},
  title        = {OmniTokenizer: {A} Joint Image-Video Tokenizer for Visual Generation},
  journal      = {CoRR},
  volume       = {abs/2406.09399},
  year         = {2024},
  url          = {https://doi.org/10.48550/arXiv.2406.09399},
  doi          = {10.48550/ARXIV.2406.09399},
  eprinttype    = {arXiv},
  eprint       = {2406.09399},
  bibsource    = {dblp computer science bibliography, https://dblp.org}
}

@inproceedings{MoCoGAN,
  author       = {Sergey Tulyakov and
                  Ming{-}Yu Liu and
                  Xiaodong Yang and
                  Jan Kautz},
  title        = {MoCoGAN: Decomposing Motion and Content for Video Generation},
  booktitle    = {2018 {IEEE} Conference on Computer Vision and Pattern Recognition,
                  {CVPR} 2018, Salt Lake City, UT, USA, June 18-22, 2018},
  pages        = {1526--1535},
  publisher    = {Computer Vision Foundation / {IEEE} Computer Society},
  year         = {2018},
  url          = {http://openaccess.thecvf.com/content\_cvpr\_2018/html/Tulyakov\_MoCoGAN\_Decomposing\_Motion\_CVPR\_2018\_paper.html},
  doi          = {10.1109/CVPR.2018.00165},
  bibsource    = {dblp computer science bibliography, https://dblp.org}
}

@inproceedings{DIGAN,
  author       = {Sihyun Yu and
                  Jihoon Tack and
                  Sangwoo Mo and
                  Hyunsu Kim and
                  Junho Kim and
                  Jung{-}Woo Ha and
                  Jinwoo Shin},
  title        = {Generating Videos with Dynamics-aware Implicit Generative Adversarial
                  Networks},
  booktitle    = {The Tenth International Conference on Learning Representations, {ICLR}
                  2022, Virtual Event, April 25-29, 2022},
  publisher    = {OpenReview.net},
  year         = {2022},
  url          = {https://openreview.net/forum?id=Czsdv-S4-w9},
  bibsource    = {dblp computer science bibliography, https://dblp.org}
}

@inproceedings{PVDM,
  author       = {Sihyun Yu and
                  Kihyuk Sohn and
                  Subin Kim and
                  Jinwoo Shin},
  title        = {Video Probabilistic Diffusion Models in Projected Latent Space},
  booktitle    = {{IEEE/CVF} Conference on Computer Vision and Pattern Recognition,
                  {CVPR} 2023, Vancouver, BC, Canada, June 17-24, 2023},
  pages        = {18456--18466},
  publisher    = {{IEEE}},
  year         = {2023},
  url          = {https://doi.org/10.1109/CVPR52729.2023.01770},
  doi          = {10.1109/CVPR52729.2023.01770},
  bibsource    = {dblp computer science bibliography, https://dblp.org}
}

@article{LVDM,
  title={Latent video diffusion models for high-fidelity long video generation},
  author={He, Yingqing and Yang, Tianyu and Zhang, Yong and Shan, Ying and Chen, Qifeng},
  journal={arXiv preprint arXiv:2211.13221},
  year={2022}
}

@article{Videogpt,
  title={Videogpt: Video generation using vq-vae and transformers},
  author={Yan, Wilson and Zhang, Yunzhi and Abbeel, Pieter and Srinivas, Aravind},
  journal={arXiv preprint arXiv:2104.10157},
  year={2021}
}

@article{Loong,
  title={Loong: Generating minute-level long videos with autoregressive language models},
  author={Wang, Yuqing and Xiong, Tianwei and Zhou, Daquan and Lin, Zhijie and Zhao, Yang and Kang, Bingyi and Feng, Jiashi and Liu, Xihui},
  journal={arXiv preprint arXiv:2410.02757},
  year={2024}
}

@inproceedings{SKY,
  author       = {Wei Xiong and
                  Wenhan Luo and
                  Lin Ma and
                  Wei Liu and
                  Jiebo Luo},
  title        = {Learning to Generate Time-Lapse Videos Using Multi-Stage Dynamic Generative
                  Adversarial Networks},
  booktitle    = {2018 {IEEE} Conference on Computer Vision and Pattern Recognition,
                  {CVPR} 2018, Salt Lake City, UT, USA, June 18-22, 2018},
  pages        = {2364--2373},
  publisher    = {Computer Vision Foundation / {IEEE} Computer Society},
  year         = {2018},
  url          = {http://openaccess.thecvf.com/content\_cvpr\_2018/html/Xiong\_Learning\_to\_Generate\_CVPR\_2018\_paper.html},
  doi          = {10.1109/CVPR.2018.00251},
  bibsource    = {dblp computer science bibliography, https://dblp.org}
}

@article{UCF,
  title={UCF101: A dataset of 101 human actions classes from videos in the wild},
  author={Soomro, K},
  journal={arXiv preprint arXiv:1212.0402},
  year={2012}
}

@inproceedings{Taichi,
  author       = {Aliaksandr Siarohin and
                  St{\'{e}}phane Lathuili{\`{e}}re and
                  Sergey Tulyakov and
                  Elisa Ricci and
                  Nicu Sebe},
  editor       = {Hanna M. Wallach and
                  Hugo Larochelle and
                  Alina Beygelzimer and
                  Florence d'Alch{\'{e}}{-}Buc and
                  Emily B. Fox and
                  Roman Garnett},
  title        = {First Order Motion Model for Image Animation},
  booktitle    = {Advances in Neural Information Processing Systems 32: Annual Conference
                  on Neural Information Processing Systems 2019, NeurIPS 2019, December
                  8-14, 2019, Vancouver, BC, Canada},
  pages        = {7135--7145},
  year         = {2019},
  bibsource    = {dblp computer science bibliography, https://dblp.org}
}

@article{FVD,
  author       = {Thomas Unterthiner and
                  Sjoerd van Steenkiste and
                  Karol Kurach and
                  Rapha{\"{e}}l Marinier and
                  Marcin Michalski and
                  Sylvain Gelly},
  title        = {Towards Accurate Generative Models of Video: {A} New Metric {\&}
                  Challenges},
  journal      = {CoRR},
  volume       = {abs/1812.01717},
  year         = {2018},
  url          = {http://arxiv.org/abs/1812.01717},
  eprinttype    = {arXiv},
  eprint       = {1812.01717},
  bibsource    = {dblp computer science bibliography, https://dblp.org}
}

@article{PSNR,
  author       = {Zhou Wang and
                  Alan C. Bovik and
                  Hamid R. Sheikh and
                  Eero P. Simoncelli},
  title        = {Image quality assessment: from error visibility to structural similarity},
  journal      = {{IEEE} Trans. Image Process.},
  volume       = {13},
  number       = {4},
  pages        = {600--612},
  year         = {2004},
  url          = {https://doi.org/10.1109/TIP.2003.819861},
  doi          = {10.1109/TIP.2003.819861},
  bibsource    = {dblp computer science bibliography, https://dblp.org}
}

@inproceedings{OpticalFlow,
  author       = {Gunnar Farneb{\"{a}}ck},
  editor       = {Josef Big{\"{u}}n and
                  Tomas Gustavsson},
  title        = {Two-Frame Motion Estimation Based on Polynomial Expansion},
  booktitle    = {Image Analysis, 13th Scandinavian Conference, {SCIA} 2003, Halmstad,
                  Sweden, June 29 - July 2, 2003, Proceedings},
  series       = {Lecture Notes in Computer Science},
  volume       = {2749},
  pages        = {363--370},
  publisher    = {Springer},
  year         = {2003},
  url          = {https://doi.org/10.1007/3-540-45103-X\_50},
  doi          = {10.1007/3-540-45103-X\_50},
  bibsource    = {dblp computer science bibliography, https://dblp.org}
}

@article{confiner,
  title={Training-free Long Video Generation with Chain of Diffusion Model Experts},
  author={Li, Wenhao and Cao, Yichao and Su, Xiu and Lin, Xi and You, Shan and Zheng, Mingkai and Chen, Yi and Xu, Chang},
  journal={arXiv preprint arXiv:2408.13423},
  year={2024}
}

@article{videolcm,
  author       = {Xiang Wang and
                  Shiwei Zhang and
                  Han Zhang and
                  Yu Liu and
                  Yingya Zhang and
                  Changxin Gao and
                  Nong Sang},
  title        = {VideoLCM: Video Latent Consistency Model},
  journal      = {CoRR},
  volume       = {abs/2312.09109},
  year         = {2023},
  url          = {https://doi.org/10.48550/arXiv.2312.09109},
  doi          = {10.48550/ARXIV.2312.09109},
  eprinttype    = {arXiv},
  eprint       = {2312.09109},
  bibsource    = {dblp computer science bibliography, https://dblp.org}
}

@article{stablevideodiffusion,
  author       = {Andreas Blattmann and
                  Tim Dockhorn and
                  Sumith Kulal and
                  Daniel Mendelevitch and
                  Maciej Kilian and
                  Dominik Lorenz and
                  Yam Levi and
                  Zion English and
                  Vikram Voleti and
                  Adam Letts and
                  Varun Jampani and
                  Robin Rombach},
  title        = {Stable Video Diffusion: Scaling Latent Video Diffusion Models to Large
                  Datasets},
  journal      = {CoRR},
  volume       = {abs/2311.15127},
  year         = {2023},
  url          = {https://doi.org/10.48550/arXiv.2311.15127},
  doi          = {10.48550/ARXIV.2311.15127},
  eprinttype    = {arXiv},
  eprint       = {2311.15127},
  bibsource    = {dblp computer science bibliography, https://dblp.org}
}

@article{show1,
  author       = {David Junhao Zhang and
                  Jay Zhangjie Wu and
                  Jia{-}Wei Liu and
                  Rui Zhao and
                  Lingmin Ran and
                  Yuchao Gu and
                  Difei Gao and
                  Mike Zheng Shou},
  title        = {Show-1: Marrying Pixel and Latent Diffusion Models for Text-to-Video
                  Generation},
  journal      = {CoRR},
  volume       = {abs/2309.15818},
  year         = {2023},
  url          = {https://doi.org/10.48550/arXiv.2309.15818},
  doi          = {10.48550/ARXIV.2309.15818},
  eprinttype    = {arXiv},
  eprint       = {2309.15818},
  bibsource    = {dblp computer science bibliography, https://dblp.org}
}

@article{videofactory,
  author       = {Wenjing Wang and
                  Huan Yang and
                  Zixi Tuo and
                  Huiguo He and
                  Junchen Zhu and
                  Jianlong Fu and
                  Jiaying Liu},
  title        = {VideoFactory: Swap Attention in Spatiotemporal Diffusions for Text-to-Video
                  Generation},
  journal      = {CoRR},
  volume       = {abs/2305.10874},
  year         = {2023},
  url          = {https://doi.org/10.48550/arXiv.2305.10874},
  doi          = {10.48550/ARXIV.2305.10874},
  eprinttype    = {arXiv},
  eprint       = {2305.10874},
  bibsource    = {dblp computer science bibliography, https://dblp.org}
}

@article{latte,
  author       = {Xin Ma and
                  Yaohui Wang and
                  Gengyun Jia and
                  Xinyuan Chen and
                  Ziwei Liu and
                  Yuan{-}Fang Li and
                  Cunjian Chen and
                  Yu Qiao},
  title        = {Latte: Latent Diffusion Transformer for Video Generation},
  journal      = {CoRR},
  volume       = {abs/2401.03048},
  year         = {2024},
  url          = {https://doi.org/10.48550/arXiv.2401.03048},
  doi          = {10.48550/ARXIV.2401.03048},
  eprinttype    = {arXiv},
  eprint       = {2401.03048},
  bibsource    = {dblp computer science bibliography, https://dblp.org}
}

@inproceedings{cogvideo,
  author       = {Wenyi Hong and
                  Ming Ding and
                  Wendi Zheng and
                  Xinghan Liu and
                  Jie Tang},
  title        = {CogVideo: Large-scale Pretraining for Text-to-Video Generation via
                  Transformers},
  booktitle    = {The Eleventh International Conference on Learning Representations,
                  {ICLR} 2023, Kigali, Rwanda, May 1-5, 2023},
  publisher    = {OpenReview.net},
  year         = {2023},
  url          = {https://openreview.net/forum?id=rB6TpjAuSRy},
  bibsource    = {dblp computer science bibliography, https://dblp.org}
}

@inproceedings{phenaki,
  author       = {Ruben Villegas and
                  Mohammad Babaeizadeh and
                  Pieter{-}Jan Kindermans and
                  Hernan Moraldo and
                  Han Zhang and
                  Mohammad Taghi Saffar and
                  Santiago Castro and
                  Julius Kunze and
                  Dumitru Erhan},
  title        = {Phenaki: Variable Length Video Generation from Open Domain Textual
                  Descriptions},
  booktitle    = {The Eleventh International Conference on Learning Representations,
                  {ICLR} 2023, Kigali, Rwanda, May 1-5, 2023},
  publisher    = {OpenReview.net},
  year         = {2023},
  url          = {https://openreview.net/forum?id=vOEXS39nOF},
  bibsource    = {dblp computer science bibliography, https://dblp.org}
}

@inproceedings{instructvideo,
  author       = {Hangjie Yuan and
                  Shiwei Zhang and
                  Xiang Wang and
                  Yujie Wei and
                  Tao Feng and
                  Yining Pan and
                  Yingya Zhang and
                  Ziwei Liu and
                  Samuel Albanie and
                  Dong Ni},
  title        = {InstructVideo: Instructing Video Diffusion Models with Human Feedback},
  booktitle    = {{IEEE/CVF} Conference on Computer Vision and Pattern Recognition,
                  {CVPR} 2024, Seattle, WA, USA, June 16-22, 2024},
  pages        = {6463--6474},
  publisher    = {{IEEE}},
  year         = {2024},
  url          = {https://doi.org/10.1109/CVPR52733.2024.00618},
  doi          = {10.1109/CVPR52733.2024.00618},
  bibsource    = {dblp computer science bibliography, https://dblp.org}
}

@article{modelscope,
  author       = {Jiuniu Wang and
                  Hangjie Yuan and
                  Dayou Chen and
                  Yingya Zhang and
                  Xiang Wang and
                  Shiwei Zhang},
  title        = {ModelScope Text-to-Video Technical Report},
  journal      = {CoRR},
  volume       = {abs/2308.06571},
  year         = {2023},
  url          = {https://doi.org/10.48550/arXiv.2308.06571},
  doi          = {10.48550/ARXIV.2308.06571},
  eprinttype    = {arXiv},
  eprint       = {2308.06571},
  bibsource    = {dblp computer science bibliography, https://dblp.org}
}

@article{DiCoDe2023,
  author       = {Yizhuo Li and
                  Yuying Ge and
                  Yixiao Ge and
                  Ping Luo and
                  Ying Shan},
  title        = {DiCoDe: Diffusion-Compressed Deep Tokens for Autoregressive Video
                  Generation with Language Models},
  journal      = {CoRR},
  volume       = {abs/2412.04446},
  year         = {2024},
  url          = {https://doi.org/10.48550/arXiv.2412.04446},
  doi          = {10.48550/ARXIV.2412.04446},
  eprinttype    = {arXiv},
  eprint       = {2412.04446},
  bibsource    = {dblp computer science bibliography, https://dblp.org}
}

@article{Multimodal2023,
  author       = {Yutao Sun and
                  Hangbo Bao and
                  Wenhui Wang and
                  Zhiliang Peng and
                  Li Dong and
                  Shaohan Huang and
                  Jianyong Wang and
                  Furu Wei},
  title        = {Multimodal Latent Language Modeling with Next-Token Diffusion},
  journal      = {CoRR},
  volume       = {abs/2412.08635},
  year         = {2024},
  url          = {https://doi.org/10.48550/arXiv.2412.08635},
  doi          = {10.48550/ARXIV.2412.08635},
  eprinttype    = {arXiv},
  eprint       = {2412.08635},
  bibsource    = {dblp computer science bibliography, https://dblp.org}
}

@article{MovieGen2023,
  author       = {Adam Polyak and
                  Amit Zohar and
                  Andrew Brown and
                  Andros Tjandra and
                  Animesh Sinha and
                  Ann Lee and
                  Apoorv Vyas and
                  Bowen Shi and
                  Chih{-}Yao Ma and
                  Ching{-}Yao Chuang and
                  David Yan and
                  Dhruv Choudhary and
                  Dingkang Wang and
                  Geet Sethi and
                  Guan Pang and
                  Haoyu Ma and
                  Ishan Misra and
                  Ji Hou and
                  Jialiang Wang and
                  Kiran Jagadeesh and
                  Kunpeng Li and
                  Luxin Zhang and
                  Mannat Singh and
                  Mary Williamson and
                  Matt Le and
                  Matthew Yu and
                  Mitesh Kumar Singh and
                  Peizhao Zhang and
                  Peter Vajda and
                  Quentin Duval and
                  Rohit Girdhar and
                  Roshan Sumbaly and
                  Sai Saketh Rambhatla and
                  Sam S. Tsai and
                  Samaneh Azadi and
                  Samyak Datta and
                  Sanyuan Chen and
                  Sean Bell and
                  Sharadh Ramaswamy and
                  Shelly Sheynin and
                  Siddharth Bhattacharya and
                  Simran Motwani and
                  Tao Xu and
                  Tianhe Li and
                  Tingbo Hou and
                  Wei{-}Ning Hsu and
                  Xi Yin and
                  Xiaoliang Dai and
                  Yaniv Taigman and
                  Yaqiao Luo and
                  Yen{-}Cheng Liu and
                  Yi{-}Chiao Wu and
                  Yue Zhao and
                  Yuval Kirstain and
                  Zecheng He and
                  Zijian He and
                  Albert Pumarola and
                  Ali K. Thabet and
                  Artsiom Sanakoyeu and
                  Arun Mallya and
                  Baishan Guo and
                  Boris Araya and
                  Breena Kerr and
                  Carleigh Wood and
                  Ce Liu and
                  Cen Peng and
                  Dmitry Vengertsev and
                  Edgar Sch{\"{o}}nfeld and
                  Elliot Blanchard and
                  Felix Juefei{-}Xu and
                  Fraylie Nord and
                  Jeff Liang and
                  John Hoffman and
                  Jonas Kohler and
                  Kaolin Fire and
                  Karthik Sivakumar and
                  Lawrence Chen and
                  Licheng Yu and
                  Luya Gao and
                  Markos Georgopoulos and
                  Rashel Moritz and
                  Sara K. Sampson and
                  Shikai Li and
                  Simone Parmeggiani and
                  Steve Fine and
                  Tara Fowler and
                  Vladan Petrovic and
                  Yuming Du},
  title        = {Movie Gen: {A} Cast of Media Foundation Models},
  journal      = {CoRR},
  volume       = {abs/2410.13720},
  year         = {2024},
  url          = {https://doi.org/10.48550/arXiv.2410.13720},
  doi          = {10.48550/ARXIV.2410.13720},
  eprinttype    = {arXiv},
  eprint       = {2410.13720},
  bibsource    = {dblp computer science bibliography, https://dblp.org}
}

@article{NOVA2023,
  title={Autoregressive Video Generation without Vector Quantization},
  author={Deng, Haoge and Pan, Ting and Diao, Haiwen and Luo, Zhengxiong and Cui, Yufeng and Lu, Huchuan and Shan, Shiguang and Qi, Yonggang and Wang, Xinlong},
  journal={arXiv preprint arXiv:2412.14169},
  year={2024}
}

@inproceedings{
wang2025larp,
title={{LARP}: Tokenizing Videos with a Learned Autoregressive Generative Prior},
author={Hanyu Wang and Saksham Suri and Yixuan Ren and Hao Chen and Abhinav Shrivastava},
booktitle={The Thirteenth International Conference on Learning Representations},
year={2025},
url={https://openreview.net/forum?id=Wr3UuEx72f}
}

@article{ren2025next,
  title={Next block prediction: Video generation via semi-autoregressive modeling},
  author={Ren, Shuhuai and Ma, Shuming and Sun, Xu and Wei, Furu},
  journal={arXiv preprint arXiv:2502.07737},
  year={2025}
}

@inproceedings{zhou2024visual,
  title={Visual in-context learning for large vision-language models},
  author={Zhou, Yucheng and Li, Xiang and Wang, Qianning and Shen, Jianbing},
  booktitle={Findings of the Association for Computational Linguistics: ACL 2024},
  pages={15890--15902},
  year={2024}
}

@inproceedings{zhou2025weak,
  title={Weak to strong generalization for large language models with multi-capabilities},
  author={Zhou, Yucheng and Shen, Jianbing and Cheng, Yu},
  booktitle={The Thirteenth International Conference on Learning Representations},
  year={2025}
}

@article{zhou2023thread,
  title={Thread of thought unraveling chaotic contexts},
  author={Zhou, Yucheng and Geng, Xiubo and Shen, Tao and Tao, Chongyang and Long, Guodong and Lou, Jian-Guang and Shen, Jianbing},
  journal={arXiv preprint arXiv:2311.08734},
  year={2023}
}

@article{wang2026ladr,
  title={LADR: Locality-Aware Dynamic Rescue for Efficient Text-to-Image Generation with Diffusion Large Language Models},
  author={Wang, Chenglin and Zhou, Yucheng and Chen, Shawn and Wang, Tao and Zhang, Kai},
  journal={arXiv preprint arXiv:2603.13450},
  year={2026}
}

@article{yang2025hicogen,
  title={HiCoGen: Hierarchical Compositional Text-to-Image Generation in Diffusion Models via Reinforcement Learning},
  author={Yang, Hongji and Zhou, Yucheng and Han, Wencheng and Tao, Runzhou and Qiu, Zhongying and Yang, Jianfei and Shen, Jianbing},
  journal={arXiv preprint arXiv:2511.19965},
  year={2025}
}

@inproceedings{yang2025dc,
  title={Dc-controlnet: Decoupling inter-and intra-element conditions in image generation with diffusion models},
  author={Yang, Hongji and Han, Wencheng and Zhou, Yucheng and Shen, Jianbing},
  booktitle={Proceedings of the IEEE/CVF International Conference on Computer Vision},
  pages={19065--19074},
  year={2025}
}

@inproceedings{chen2025towards,
  title={Towards stabilized and efficient diffusion transformers through long-skip-connections with spectral constraints},
  author={Chen, Guanjie and Zhao, Xinyu and Zhou, Yucheng and Qu, Xiaoye and Chen, Tianlong and Cheng, Yu},
  booktitle={Proceedings of the IEEE/CVF International Conference on Computer Vision},
  pages={17708--17718},
  year={2025}
}

@inproceedings{yang2025self,
  title={Self-rewarding large vision-language models for optimizing prompts in text-to-image generation},
  author={Yang, Hongji and Zhou, Yucheng and Han, Wencheng and Shen, Jianbing},
  booktitle={Findings of the Association for Computational Linguistics: ACL 2025},
  pages={7332--7349},
  year={2025}
}

@inproceedings{songbroad,
  title={From Broad Exploration to Stable Synthesis: Entropy-Guided Optimization for Autoregressive Image Generation},
  author={Song, Han and Zhou, Yucheng and Shen, Jianbing and Cheng, Yu},
  booktitle={The Fourteenth International Conference on Learning Representations},
  year={2026}
}

@inproceedings{zhoucondition,
  title={Condition Errors Refinement in Autoregressive Image Generation with Diffusion Loss},
  author={Zhou, Yucheng and Li, Hao and Shen, Jianbing},
  booktitle={The Fourteenth International Conference on Learning Representations},
  year={2026}
}
\clearpage
\appendix
\section{Model Details}\label{app:details}
We employ OmniTokenizer~\citep{OmniTokenizer}, a VQVAE for video, to transform video frames into discrete tokens for autoregressive language model training. 
However, the sequence length of the model trained on fewer frames is limited to 2048 tokens, while the original 17-frame video would typically be encoded as a sequence of 5120 tokens.
In this setting, we train the Fewer-Frames model on the reduced token sequence, and during inference, we first generate an initial 2048 tokens. 
Then, the last 1024 tokens of the previously generated sequence are used to produce the next 1024 tokens, and this process is repeated until the full 5120-token sequence is generated.
The resulting token sequence is subsequently decoded back into a 17-frame video.
The Baseline model, in contrast, is trained directly on the full 5120-token sequence.

\section{Proof of Proposition~\ref{prop:error_accumulation}}
\label{app:the1}

\begin{proof}
Let $\mathcal{M}_{Base}$ and $\mathcal{M}_{FF}$ denote the Baseline and Fewer-Frames models, respectively. Let the ground-truth sequence of token blocks be $\mathbf{T} = (\mathbf{T}_1, \dots, \mathbf{T}_K)$. The generated sequences are $\hat{\mathbf{T}}_{Base}$ and $\hat{\mathbf{T}}_{FF}$. We aim to prove that the expected total error of the Fewer-Frames model, $\mathbb{E}[\|\mathbf{T} - \hat{\mathbf{T}}_{FF}\|]$, is greater than or equal to that of the Baseline model, $\mathbb{E}[\|\mathbf{T} - \hat{\mathbf{T}}_{Base}\|]$.

Our proof is based on analyzing the error accumulation at each generation step, which stems from two main factors: (1) the information available for prediction and (2) the propagation of errors from previous steps (i.e., exposure bias).

\paragraph{1. One-Step Prediction Error with Perfect History.}
First, consider the ideal scenario of predicting block $\mathbf{T}_k$ given a perfect history of ground-truth blocks. The Baseline model uses a long context $\mathbf{T}_{<k} = (\mathbf{T}_1, \dots, \mathbf{T}_{k-1})$, while the Fewer-Frames model uses only a short context $\mathbf{T}_{k-1}$.
\begin{align}
\hat{\mathbf{T}}_k^{Base} &= \mathcal{M}_{Base}(\mathbf{T}_{<k}) \\
\hat{\mathbf{T}}_k^{FF} &= \mathcal{M}_{FF}(\mathbf{T}_{k-1})
\end{align}
The sequence $\mathbf{T}_{<k}$ contains strictly more information about the data distribution for predicting $\mathbf{T}_k$ than $\mathbf{T}_{k-1}$ alone. A model conditioned on more relevant information is expected to have a lower prediction error. Therefore, the one-step prediction error for the Baseline model is expected to be lower:
\begin{align}
\mathbb{E}[\|\mathbf{T}_k - \hat{\mathbf{T}}_k^{Base}\|] \le \mathbb{E}[\|\mathbf{T}_k - \hat{\mathbf{T}}_k^{FF}\|]
\label{eq:one_step_error}
\end{align}
This establishes the inherent modeling advantage of $\mathcal{M}_{Base}$ due to its access to a richer context.

\paragraph{2. Error Propagation under Exposure Bias.}
During actual inference, models are conditioned on their own previously generated, potentially erroneous outputs. Let $\mathbf{E}_{<k} = \mathbf{T}_{<k} - \hat{\mathbf{T}}_{<k}$ be the cumulative error up to step $k$.

For the Fewer-Frames model, the input for generating the $k$-th block is $\hat{\mathbf{T}}_{k-1} = \mathbf{T}_{k-1} - \mathbf{E}_{k-1}$. The model's output is $\mathcal{M}_{FF}(\mathbf{T}_{k-1} - \mathbf{E}_{k-1})$. Because $\mathcal{M}_{FF}$ was trained only on local transitions, it lacks the global context necessary to correct for drift. Any error $\mathbf{E}_{k-1}$ in its input directly and significantly perturbs the generation of the next block, causing errors to compound at each step.

For the Baseline model, the input is the full noisy history $\hat{\mathbf{T}}_{<k} = \mathbf{T}_{<k} - \mathbf{E}_{<k}$. Although this history is also imperfect, the model $\mathcal{M}_{Base}$ has been trained on long-range dependencies. The information contained in the earlier parts of the history ($\hat{\mathbf{T}}_{<k-1}$) can help the model mitigate the impact of recent errors in $\hat{\mathbf{T}}_{k-1}$. It is more robust to local perturbations because it can leverage a wider context to stay on the true data manifold.

Consequently, the error added at step $k$ is amplified more severely in the Fewer-Frames model due to its sensitivity to the error in its limited context. Let $\delta_k(\mathbf{E}_{<k})$ be the additional error introduced at step $k$ as a function of the previous cumulative error $\mathbf{E}_{<k}$. We argue that:
\begin{align}
\mathbb{E}[\|\delta_k^{FF}(\mathbf{E}_{<k})\|] \ge \mathbb{E}[\|\delta_k^{Base}(\mathbf{E}_{<k})\|]
\end{align}

The total error for each model is the accumulation of errors introduced at each step. The Fewer-Frames model starts with a higher intrinsic one-step prediction error (Eq.~\ref{eq:one_step_error}) and suffers from a more severe error propagation mechanism at each subsequent step. The combination of these two factors leads to a larger total accumulated error.
\begin{align}
&\mathbb{E}[\|\mathbf{E}_{FF}\|] = \mathbb{E}\left[\left\|\sum_{k=1}^K \delta_k^{FF}\right\|\right] \notag\\
\ge &\mathbb{E}\left[\left\|\sum_{k=1}^K \delta_k^{Base}\right\|\right] = \mathbb{E}[\|\mathbf{E}_{Base}\|]
\end{align}
\end{proof}

\section{Proof and Explanation for Proposition~\ref{prop:local_opt_error}}
\label{app:the2}

Proposition~\ref{prop:local_opt_error} states that the Local-Opt. model achieves a lower expected error within a window than the Fewer-Frames model. This advantage stems from two core aspects of its design: its optimization objective and the use of overlapping windows.

\subsection{Formal Argument for Error Minimization}
\begin{proof}
Let $\theta$ be the model parameters. The training objective of Local-Opt., as defined in Eq.~\ref{eq:local_loss}, is to minimize the expected negative log-likelihood over all possible windows $\mathcal{W}$:
\begin{align}
\mathcal{L}_{LO}(\theta) = \mathbb{E}_{\mathcal{W}} \left[ - \sum_{i=s}^{s+W-1} \log P(\mathbf{e}_i | \mathbf{E}_{<i}; \theta) \right]
\label{eq:lo_objective}
\end{align}
The resulting trained model, $\theta_{LO}$, is by definition the minimizer of this objective:
\begin{align}
\theta_{LO} = \arg\min_{\theta} \mathcal{L}_{LO}(\theta)
\end{align}
The Fewer-Frames model, $\theta_{FF}$, is trained on a different objective, $\mathcal{L}_{FF}$, which optimizes predictions on isolated, short sequences without access to true, long-range context. Since $\theta_{FF}$ is not optimized for the $\mathcal{L}_{LO}$ objective, its parameters are suboptimal for this loss function. Therefore, it follows that:
\begin{align}
\mathcal{L}_{LO}(\theta_{LO}) \leq \mathcal{L}_{LO}(\theta_{FF})
\label{eq:loss_inequality}
\end{align}
Minimizing the negative log-likelihood is a standard and effective surrogate for minimizing the prediction error (e.g., Mean Squared Error in the latent space). Thus, the inequality in Eq.~\ref{eq:loss_inequality} directly implies that the expected prediction error of the Local-Opt. model within any given window is lower than or equal to that of the Fewer-Frames model. This formally proves the proposition.
\end{proof}

\subsection{Contribution of Overlapping Windows}
The use of overlapping windows ($S < W$) provides an implicit mechanism for iterative refinement, which further reduces error by enhancing temporal consistency. We can formalize this by examining the gradient received by the parameters responsible for a single token $\mathbf{e}_i$.

Let $\mathcal{C}(i)$ be the set of all starting indices $s$ for windows $\mathcal{W}_s$ that contain the token $\mathbf{e}_i$. Due to the overlap, $|\mathcal{C}(i)| > 1$ for many tokens. The total loss term concerning $\mathbf{e}_i$ can be conceptualized as a sum over these windows:
\begin{align}
\mathcal{L}(\mathbf{e}_i; \theta) \propto \sum_{s \in \mathcal{C}(i)} -\log P(\mathbf{e}_i | \mathbf{E}_{<i}; \theta)
\end{align}
Consequently, the gradient update for $\theta$ with respect to the prediction of $\mathbf{e}_i$ is an aggregation of signals from different contexts:
\begin{align}
\nabla_{\theta} \mathcal{L}(\mathbf{e}_i) \propto \sum_{s \in \mathcal{C}(i)} \nabla_{\theta} \left( -\log P(\mathbf{e}_i | \mathbf{E}_{<i}; \theta) \right)
\end{align}
This averaging process forces the model to learn a representation for $\mathbf{e}_i$ that is valid and consistent across multiple preceding contexts ($\mathbf{E}_{<s_1}, \mathbf{E}_{<s_2}, \dots$). In contrast, the Fewer-Frames model learns from only a single, fixed context for each segment. This multi-context optimization makes the Local-Opt. model more robust and less prone to local errors, contributing to the overall error reduction stated in the proposition.

\section{Algorithm for Local-Opt. Training}
\label{appendix:local_opt_algorithm}
\begin{algorithm}[ht]\small
\caption{\small Local-Opt. Training}
\label{alg:local_optimization_appendix}
\begin{algorithmic}[1]
\STATE {\bfseries Input:} Full token sequence $\mathbf{E} = (\mathbf{e}_1, \mathbf{e}_2, \dots, \mathbf{e}_N)$, Window size $W$, Stride size $S = W/2$, Training iteration Number $I$
\FOR{$i=1$ {\bfseries to} $I$}
    \STATE Randomly select a starting index $s$ from $\{1 + k \cdot S \mid k \in \mathbb{Z}_{\ge 0}, 1 + k \cdot S \le N - W + 1 \}$
    \STATE Define the optimization window: $\mathcal{W} = \{s, s+1, \dots, s+W-1\}$
    \STATE Input sequence to the model: $\mathbf{X}_{in} = \mathbf{E}_{[:\max(\mathcal{W})]}$
    \STATE Model prediction for the window: $\hat{\mathbf{E}}_{LO,\mathcal{W}} = \text{Model}(\mathbf{X}_{in})_{[\min(\mathcal{W}):\max(\mathcal{W})]}$
    \STATE Calculate the loss: $\mathcal{L}(\mathbf{E}_{\mathcal{W}}, \hat{\mathbf{E}}_{LO,\mathcal{W}})$
    \STATE Update model parameters based on $\theta \leftarrow \theta - \eta \cdot \nabla_{\theta} \mathcal{L}$, with gradients masked outside the window $\mathcal{W}$
\ENDFOR
\end{algorithmic}
\end{algorithm}

\section{Further Analysis}
\label{sec:further_analysis}

\begin{table*}[!t]
\centering
\caption{FVD under different motion dynamics on UCF101.}
\label{tab:motion}
\begin{tabular}{lcccc}
\toprule
Motion Group & Avg. Optical Flow & Baseline & ReCo (Ours) & Improvement \\
\midrule
Low Motion & $<3.0$ & 215.4 & 204.6 & +5.0\% \\
Medium Motion & $3.0{-}6.0$ & 542.7 & \textbf{390.2} & \textbf{+28.1\%} \\
High Motion & $>6.0$ & 985.2 & 815.7 & +17.2\% \\
\bottomrule
\end{tabular}
\end{table*}
\subsection{Robustness under Different Motion Dynamics}
\label{sec:motion}

\begin{table}[!t]\small
\centering
\caption{Sensitivity analysis of the continuity weight $\lambda$.}
\label{tab:lambda}
\resizebox{\linewidth}{!}{
\begin{tabular}{cccc}
\toprule
$\lambda$ & FFS (FVD $\downarrow$) & SKY (FVD $\downarrow$) & Note \\
\midrule
0.0 & 127.11 & 179.84 & Local-Opt only \\
0.01 & 91.42 & 105.30 &  \\
0.05 & 73.85 & 89.22 &  \\
\textbf{0.1} & \textbf{72.60} & \textbf{87.50} & Default \\
0.5 & 76.93 & 94.15 &  \\
1.0 & 88.24 & 112.67 & Over-regularized \\
\bottomrule
\end{tabular}}
\end{table}
\begin{table}[!t]\small
\centering
\caption{Effect of window overlap on performance and training speed.}
\label{tab:overlap}
\begin{tabular}{lcccc}
\toprule
Overlap & FFS (FVD $\downarrow$) & SKY (FVD $\downarrow$) & Speedup \\
\midrule
0\% & 98.4 & 112.1 & $>2.0\times$ \\
50\% & \textbf{72.6} & \textbf{87.5} & $\sim2.0\times$ \\
75\% & 71.9 & 86.9 & $<2.0\times$ \\
\bottomrule
\end{tabular}
\end{table}
\begin{table}[!t]\small
\centering
\caption{Long video generation results on SkyTimelapse.}
\label{tab:long}
\begin{tabular}{lcc}
\toprule
Method & 32 Frames & 64 Frames \\
\midrule
Baseline (Full Context) & 83.5 & 79.2 \\
ReCo (Ours) & \textbf{78.8} & \textbf{71.1} \\
\bottomrule
\end{tabular}
\end{table}

To examine whether the Lipschitz-inspired continuity constraint remains effective under different temporal dynamics, we analyze model performance across videos with varying motion magnitudes.
Experiments are conducted on the UCF101 test set, where videos are grouped based on their average optical flow magnitude into low-, medium-, and high-motion regimes.

ReCo consistently outperforms the baseline across all motion regimes.
The most substantial improvement occurs in the medium-motion setting, indicating a sweet spot where continuity constraints most effectively suppress error propagation.
Importantly, even under high-motion scenarios, ReCo significantly mitigates model collapse, whereas the baseline exhibits severe error explosion.

\subsection{Sensitivity to the Continuity Weight $\lambda$}
\label{sec:lambda}

We study the sensitivity of ReCo to the continuity loss weight $\lambda$ in Eq.~(11).
Experiments are conducted on the FFS and SKY datasets by varying $\lambda$ while keeping other hyperparameters fixed.
Performance remains stable across a wide range of $\lambda$ values from 0.01 to 0.5, and all settings substantially outperform the variant without continuity regularization.
This indicates that ReCo is not sensitive to precise hyperparameter tuning.

\subsection{Effect of Window Overlap}
\label{sec:overlap}

We further analyze the effect of window overlap, which controls the trade-off between training speed and temporal consistency.
A larger overlap increases context sharing across optimization windows but reduces training efficiency.
An overlap of 50\% provides the best balance, achieving approximately a $2\times$ training speedup while maintaining strong temporal consistency.

\begin{figure}[!t]
    \centering
    \includegraphics[width=\linewidth]{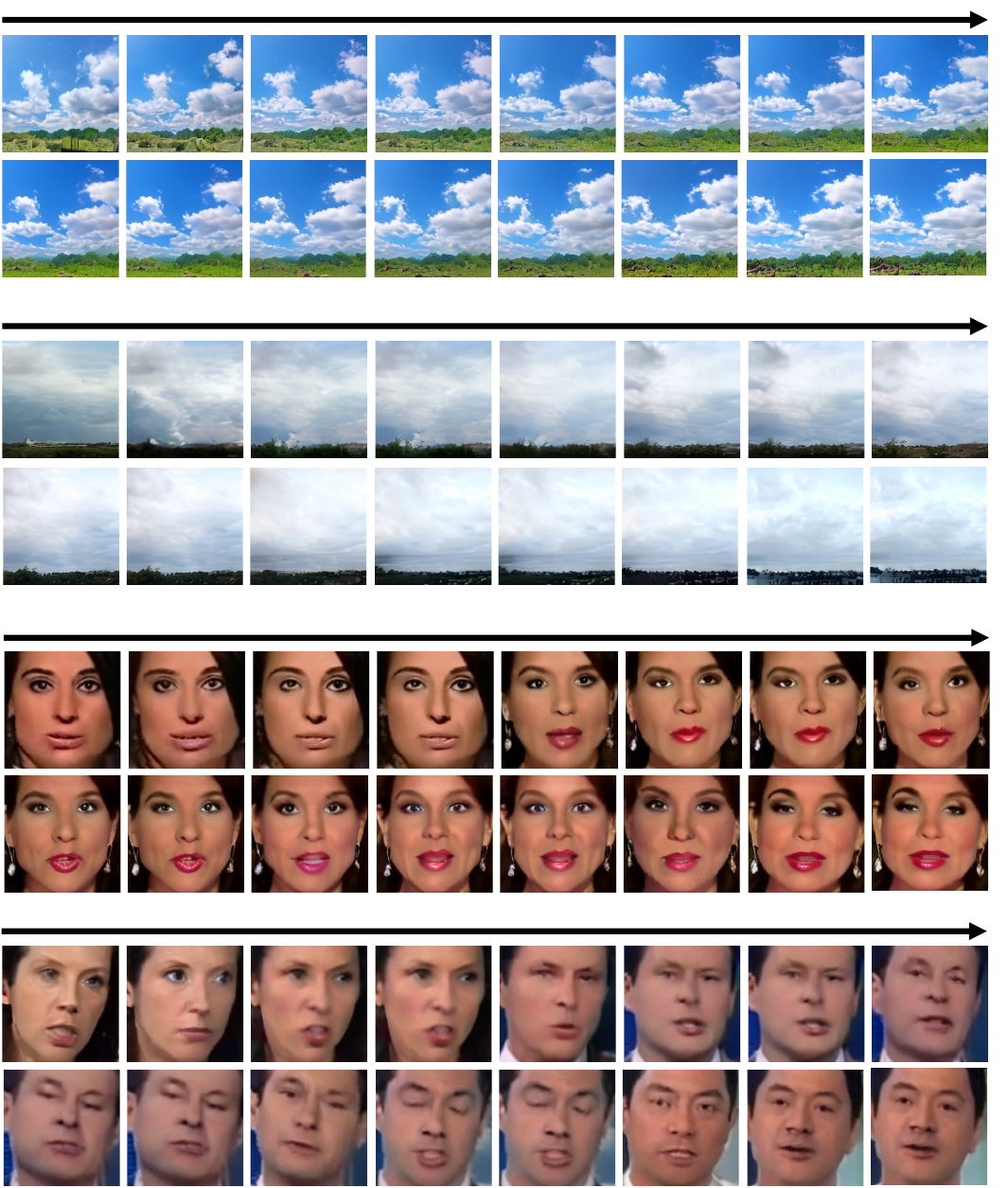}
    \caption{\small More Error Cases from Fewer-Frames model.}
    \label{fig:error_case}
\end{figure}

\subsection{Long Video Generation}
\label{sec:long_video}

Finally, we evaluate the stability of ReCo for longer video generation.
Experiments are conducted on the SkyTimelapse dataset with sequence lengths of 32 and 64 frames.

As the sequence length increases, the performance gap between ReCo and the baseline becomes larger, indicating that the continuity constraint effectively suppresses long-horizon error accumulation.

\section{More Error Cases from Fewer-Frames Model}\label{appendix:cases}
We visualize additional error cases generated by the Fewer-Frames model in Figure \ref{fig:error_case}.

\end{document}